\documentclass{article}

\IfFileExists{neurips_2026.sty}{%
  \usepackage[preprint]{neurips_2026}%
}{%
  \usepackage[margin=1in]{geometry}%
  \providecommand{\And}{\hspace{2em}}%
  \newenvironment{ack}{\section*{Acknowledgements}}{}%
}
\providecommand{\answerYes}[1][]{[Yes]}
\providecommand{\answerNo}[1][]{[No]}
\providecommand{\answerNA}[1][]{[NA]}
\providecommand{\answerTODO}[1][]{[TODO]}

\usepackage[utf8]{inputenc}
\usepackage[T1]{fontenc}
\usepackage{microtype}
\usepackage{wrapfig}
\usepackage{graphicx}
\usepackage{tikz}  % overlay row labels on trajectory images in fig:method_overview
\usetikzlibrary{calc}  % needed for [yshift=...,xshift=...]coord.anchor syntax
\usepackage{subcaption}
\usepackage{booktabs}
\usepackage{array}  % m{...} column type for vertical centering in tabular cells
\usepackage{amsmath,amssymb,amsthm,amsfonts,mathtools}
\usepackage{nicefrac}
\usepackage{algorithm}
\usepackage{algpseudocode}
\usepackage{hyperref}
\usepackage{url}
\usepackage{rotating}
\usepackage{adjustbox}
\usepackage{caption}
\usepackage[table]{xcolor}
\usepackage{placeins}  % \FloatBarrier — keep figures within their section
\usepackage{pifont}
\providecommand{\cmark}{\ding{51}}
\providecommand{\xmark}{\ding{55}}

\makeatletter
\@ifpackageloaded{natbib}{}{\usepackage{natbib}}
\makeatother

\graphicspath{{./}{./figures/}{./images/}{./fig/}}

\let\origincludegraphics\includegraphics
\renewcommand{\includegraphics}[2][]{%
  \IfFileExists{#2}{\origincludegraphics[#1]{#2}}%
    {\fbox{\parbox{0.45\linewidth}{\centering\footnotesize\ttfamily missing: \detokenize{#2}}}}%
}

\newtheorem{proposition}{Proposition}
\newtheorem{remark}{Remark}

\title{PG-MAP: Joint MAP Optimization for Inference-Time Alignment\\of Diffusion and Flow-Matching Models}

\author{%
  Ruolan Sun \\
  Stony Brook University \\
  \texttt{ruolan.sun@stonybrook.edu}
  \And
  Pawel Polak \\
  Stony Brook University \\
  \texttt{pawel.polak@stonybrook.edu}
}
\begin{document}
\maketitle

\begin{abstract}
Inference-time alignment of pretrained text-to-image models is typically performed along a single control axis, such as classifier-free guidance, attention editing, or reward-based latent perturbations. This limitation prevents modeling joint dependencies between conditioning and latent variables and hinders transfer across generative transports. We propose PG-MAP, a training-free framework that formulates inference-time alignment as a trajectory-level Gibbs-MAP / proximal energy optimization over the conditioning $c$ and latent state $z_t$ via a forward-consistency coupling, optionally guided by a frozen preference reward. This joint formulation enables coordinated updates across modalities while remaining compatible with both diffusion and flow-matching models through transport-specific adaptations. Across diffusion backbones (SD~1.5, SDXL), PG-MAP consistently improves alignment metrics such as PickScore and Aesthetic, and can be effectively combined with tuned classifier-free guidance to achieve the strongest overall performance. On flow-matching models (SD3.5-medium), the framework reduces to a latent-only variant, achieving $\mathbf{91.9\%}$ PickScore and $75.7\%$ HPS win rates against a static baseline, with controlled experiments ruling out noise-related artifacts. Human evaluations further confirm consistent preference over strong baselines, including tuned CFG and compute-matched universal guidance. Finally, an oracle-routing analysis shows that the relative importance of conditioning and latent optimization depends on prompt types, surfacing further headroom that a per-prompt selector could exploit.

\smallskip
\noindent\textbf{Code:} \url{https://github.com/sophialanlan/PG-MAP}
\end{abstract}

%=============================================================================
\section{Introduction}
\label{sec:intro}
%=============================================================================

Diffusion and flow-matching models~\citep{ho2020ddpm,rombach2022ldm,esser2024sd3} synthesize images by iteratively denoising a latent variable conditioned at every step on a fixed text embedding $c_0{=}\tau(y)$. The same embedding drives denoising at high-noise timesteps (which resolve global layout) and low-noise timesteps (which refine local detail), with no mechanism to reflect the changing information needs of the denoiser; compositional prompts in particular suffer from attribute leakage during early denoising~\citep{chefer2023attend,hertz2022prompt2prompt}. Existing inference-time fixes act on a single axis: conditioning-side methods edit cross-attention or learn embeddings~\citep{chefer2023attend,hertz2022prompt2prompt,gal2023textualinversion,ruiz2023dreambooth,wen2024pez}, latent-side methods perturb $z_t$ along a reward gradient~\citep{bansal2023universal,yu2023freedom,benhamu2024dflow,patel2025flowchef}, and training-based alternatives~\citep{wallace2023diffusiondpo} sidestep both axes by retraining $\theta$. No prior framework couples $c$ and $z_t$ through the denoiser's own forward kernel --- what we call a \emph{forward-consistency coupling} --- so that updates on the two axes are coordinated rather than additive; nor has any been analyzed across both diffusion and flow-matching transports.

Existing methods are also \emph{static} --- fixing the control axis once at $z_T$ or offline --- whereas the trajectory itself is dynamic. We propose \textbf{PG-MAP} (Preference-Guided Adaptive MAP), a training-free framework that recasts each denoising step as a proximal MAP problem with per-step objectives, schedule-adaptive trust regions, and a step-dependent active set.

We exploit two properties of this framework. \emph{(i) Adaptive, per-step refinement of $(c, z_t)$}: rather than perturbing the initial noise $z_T$ once or learning $c$ offline as in prior work, PG-MAP re-optimizes both variables at every denoising step under a schedule-adaptive trust region, so the conditioning and the latent inform each other as the trajectory unfolds, with the prior loosening at high noise (where $z$ is malleable) and tightening near the data end (where the trajectory is fragile). \emph{(ii) One objective, two transports}: the same $\mathcal{J}_t$ instantiates on diffusion as full joint refinement and on flow matching as a transport-specific reduction to a latent-only variant we denote UG-FM, so a single framework covers both denoising paradigms. Figure~\ref{fig:pgmap_showcase} previews the headline visual claim on SDXL: a single PG-MAP run improves both the $c$-side compositional structure (body silhouette, hand pose) and the $z$-side texture / lighting (feathers, hair) over the static baseline at the same seed, lifting both axes jointly.

\paragraph{Contributions.}
\begin{itemize}\setlength{\itemsep}{1pt}
  \item \textbf{Joint $(c, z_t)$ MAP framework with forward-consistency coupling.} The first inference-time framework that couples the two axes through the denoiser's own forward kernel, targeting composition ($c$-side) and texture ($z$-side) failure modes simultaneously (Fig.~\ref{fig:pgmap_showcase}).
  \item \textbf{Unified objective covering prior single-axis methods.} $\mathcal{J}_t$ recovers conditioning-only and latent-only variants and a Universal-Guidance-style limit as analytic special cases (Rem.~\ref{rem:special_cases}); CFG modifies the denoiser vector field and is composable with PG-MAP rather than a special case of it. Joint coupling and adaptive scheduling are the axes prior single-axis methods do not exploit.
  \item \textbf{Schedule-adaptive, step-dependent trajectory optimization.} $\mathcal{J}_t$ is explicitly time-dependent with a schedule-adaptive trust region $\sigma_z(t)$ and a step-dependent active set $\mathcal{A}_t$ that selects which variables to refine at each step.
  \item \textbf{Transport-dependent active set with empirical validation.} A local perturbation analysis motivates a transport-dependent active set $\mathcal{A}_t$, with diagnostic support; PG-MAP gains $5$--$7$\,pp on SD~1.5 / SDXL (Tab.~\ref{tab:main_results}), reaches $\mathbf{91.9\%}$ / $\mathbf{75.7\%}$ PS / HPS on SD3.5-medium (Tab.~\ref{tab:fm_main}), and wins $60$--$67\%$ pairwise human preference ($100$ raters, \S\ref{sec:human_eval}).
\end{itemize}

\begin{figure}[t]
\centering
\setlength{\tabcolsep}{3pt}
\begin{tabular}{@{}>{\centering\arraybackslash}m{0.32\linewidth} >{\centering\arraybackslash}m{0.32\linewidth} m{0.27\linewidth}@{}}
  % Row 1: Phoenix
  \includegraphics[width=\linewidth]{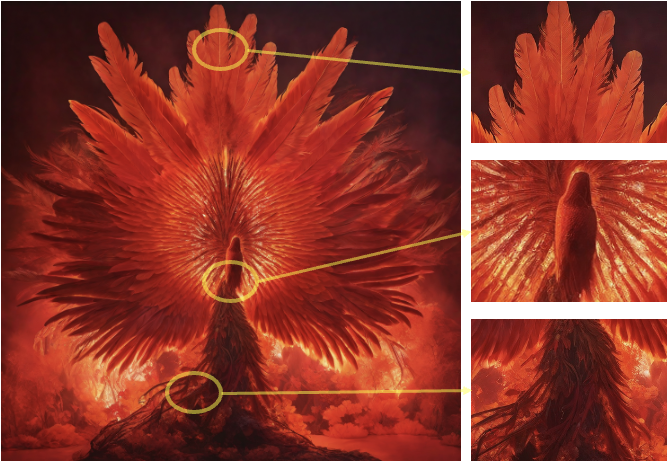} &
  \includegraphics[width=\linewidth]{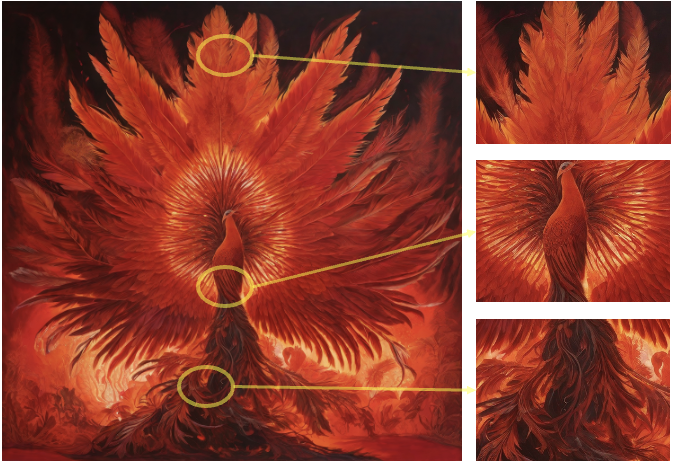} &
  {\footnotesize \emph{``a phoenix rising from ashes, vivid orange and red feathers, dramatic lighting''}\,---\,PG-MAP renders sharper \textbf{feathers} (texture, $z$-side), a more coherent \textbf{body} silhouette ($c$-side), and a richer \textbf{tail} plume.} \\[3pt]
  % Row 2: Swordsman
  \includegraphics[width=\linewidth]{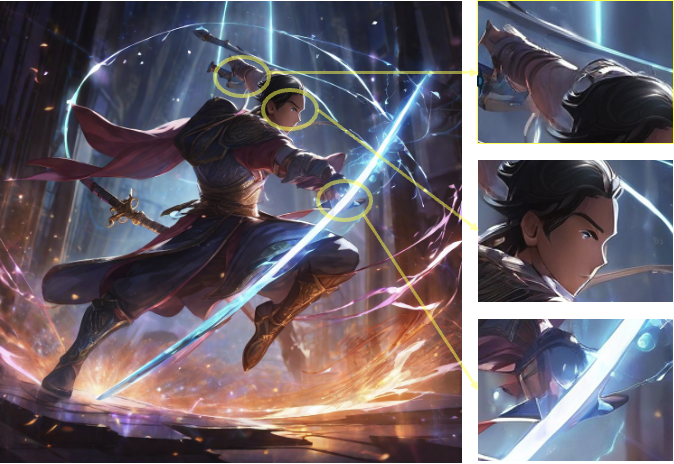} &
  \includegraphics[width=\linewidth]{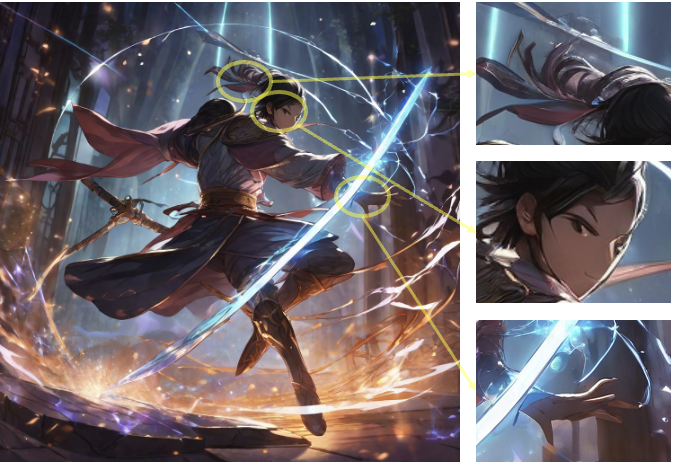} &
  {\footnotesize \emph{``a swordsman mid-leap slashing through a glowing magical barrier''}\,---\,PG-MAP produces more detailed \textbf{hair} (texture, $z$-side), a more articulated \textbf{face}, and an anatomically correct \textbf{hand} on the sword grip ($c$-side).} \\[1pt]
  % Single column-label row at the bottom
  {\footnotesize Baseline} & {\footnotesize \textbf{PG-MAP}} & \\
\end{tabular}
\caption{\textbf{Joint PG-MAP exercises both axes at once on SDXL} (same seed within each pair). Side annotations identify the per-prompt $c$- and $z$-side gains; zoom-in boxes mark them. Population-scale PartiPrompts win rates: Tab.~\ref{tab:main_results}; trajectory-level mechanism: Fig.~\ref{fig:method_overview}.}
\label{fig:pgmap_showcase}
\end{figure}

\FloatBarrier  % keep Fig.~\ref{fig:pgmap_showcase} in §1 Introduction

%=============================================================================
\section{Method}
\label{sec:method}
%=============================================================================

\subsection{Preliminaries}
\label{sec:preliminaries}

We work with a pretrained latent diffusion model~\citep{rombach2022ldm}, in which a VAE encoder $\mathcal{E}$ maps an image into a clean latent $z_0 = \mathcal{E}(x)$, a forward Gaussian process diffuses $z_0$ into pure noise $z_T$, and a learned denoiser $\epsilon_\theta(z_t, t, c)$ reverses the chain conditioned on a text embedding $c_0 = \tau(y)$. A decoder $\mathcal{D}$ maps the final clean latent back to pixel space. Concretely, the forward kernel between consecutive scheduler steps is $q(z_t \mid z_{t_\mathrm{prev}}) = \mathcal{N}(\sqrt{\alpha_t}\,z_{t_\mathrm{prev}},\,\beta_t I)$, with cumulative noise schedule $\bar\alpha_t = \prod_{i \le t} \alpha_i$. From any noisy state $z_t$ the denoiser yields the Tweedie estimate of the clean latent $\hat{z}_{0,\theta}(z_t,t,c) = (z_t - \sqrt{1-\bar\alpha_t}\,\epsilon_\theta)/\sqrt{\bar\alpha_t}$, which is the model's per-step prediction of where the trajectory is heading; the corresponding deterministic DDIM~\citep{song2021ddim} reverse step writes the next state $\hat{z}_{t_\mathrm{prev}} = \sqrt{\bar\alpha_{t_\mathrm{prev}}}\,\hat{z}_{0,\theta} + \sqrt{1-\bar\alpha_{t_\mathrm{prev}}}\,\epsilon_\theta$ as a deterministic function of $z_t$ and $c$.

Two properties of this standard pipeline matter for what follows. First, the conditioning $c_0$ is computed once from the prompt and \emph{never modified} as $t$ descends, so the same embedding drives both the high-noise steps that decide global layout and the low-noise steps that paint local texture. Second, the chain back to the clean image is fully differentiable, so any frozen evaluation function can be queried as a differentiable signal on the model's per-step preview of where the trajectory is heading. Section~\ref{sec:joint_map} turns these two observations into a per-step optimization problem over $(c, z_t)$.

\subsection{PG-MAP objective}
\label{sec:joint_map}

We treat $c$ and $z_t$ as latent variables with Gaussian anchoring priors $\mathcal{N}(c;\,\mu_t,\sigma_c^2 I)$ and $\mathcal{N}(z_t;\,z_t^{\mathrm{ddim}},\sigma_z(t)^2 I)$, anchored at the unperturbed values ($\mu_t{=}c_0$; $z_t^{\mathrm{ddim}}$ is the trajectory point before refinement). The schedule-adaptive scale $\sigma_z(t){=}\gamma\sqrt{1-\bar\alpha_t}$ tracks the marginal noise scale of the diffusion process and gives a scale-invariant trust region; isotropic anchoring is a practical default backed by a low-rank covariance diagnostic (Appendix~\ref{app:noisezoo}). With skipped DDIM steps we use the conditional coefficients $a_{t\mid s}{=}\bar\alpha_t/\bar\alpha_s$ and $\beta_{t\mid s}{=}1-a_{t\mid s}$ ($s{=}t_\mathrm{prev}$); for consecutive training steps these reduce to $\alpha_t,\beta_t$. The one-step \emph{residual} that couples $c$ and $z_t$ through the denoiser is $r_t(c,z_t){=}z_t-\sqrt{a_{t\mid s}}\,\hat{z}_{s,\theta}(z_t,t,c)$, and the reward acts on the Tweedie preview $\hat{x}_0(z_t,c){=}\mathcal{D}(\hat{z}_{0,\theta})$. The full PG-MAP energy is
\begin{equation}
\begin{aligned}
\mathcal{J}_t(c, z_t) \;=\;\;
&  \underbrace{-\tfrac{1}{2\beta_{t\mid s}}\,\bigl\| r_t(c, z_t) \bigr\|^2}_{\text{forward-consistency residual }\ell_t(c,\,z_t)} \\[4pt]
&  \underbrace{-\tfrac{1}{2\sigma_c^2}\,\| c - \mu_t \|^2 \;-\; \tfrac{1}{2\sigma_z(t)^2}\,\bigl\| z_t - z_t^{\mathrm{ddim}} \bigr\|^2}_{\text{Gaussian anchoring priors }\mathcal{R}_c(c)\,+\,\mathcal{R}_z(z_t)}
   \;+\; \underbrace{\lambda\, Q\bigl(\hat{x}_0(z_t, c),\, y\bigr)}_{\text{preference reward tilt}}.
\end{aligned}
\label{eq:pgmap_objective}
\end{equation}
Because $r_t$ depends on the optimized state $z_t$ (through the denoiser), the first factor is \emph{not} the normalized transition density $q(z_t\mid \hat{z}_{s,\theta})$; equivalently, it is the log-density of a virtual zero-residual observation $u_t{=}0$ under $u_t{\mid}c,z_t\sim\mathcal{N}(r_t,\beta_{t\mid s}I)$, which has a $(c,z_t)$-independent normalizer (Appendix~\ref{app:grads}). Together with the Gaussian anchors and the reward tilt, $\mathcal{J}_t$ defines a Gibbs-MAP energy whose normalizer is independent of the candidate point and so does not affect MAP. Beyond the time-varying $\beta_{t\mid s}, \sqrt{a_{t\mid s}}, \sigma_z(t)$, the framework treats the step-dependent active set $\mathcal{A}_t \subseteq \{c, z_t\}$ and reward gate $\lambda_t = \lambda \cdot \mathbf{1}[t/T > 1-\rho_Q]$ as explicit hyperparameters whose optimal form flips between transports (\S\ref{sec:fm_extension}). CFG and PG-MAP act on different control surfaces: CFG modifies the denoiser vector field at a fixed query by mixing conditional and unconditional predictions, while PG-MAP moves the query point $(c, z_t)$ under a fixed denoiser and proximal energy. They are therefore composable, as Tuned-CFG\,$+$\,PG-MAP demonstrates empirically; CFG is not a special case of $\mathcal{J}_t$. The refined pair is $(c_t^\star, z_t^\star) = \arg\max \mathcal{J}_t$.

Figure~\ref{fig:method_overview} visualizes two specializations of $\mathcal{J}_t$ on SDXL: (a) MAP-$c$ recovers the prompt-subject identity (panda); (b) Reward-$z$ enriches local texture (galaxy). The displacement traces (c, d) reflect the framework's asymmetric prior design: constant $\sigma_c$ gives $\|c_t^\star - c_0\|$ that \emph{grows} toward the data end as the cross-attention signal sharpens (empirical $L_c$ in App.~\ref{app:proofs}); schedule-adaptive $\sigma_z(t){=}\gamma\sqrt{1-\bar\alpha_t}$ gives $\|z_t^\star - z_t^{\rm ddim}\|$ that \emph{decays} as the trust region tightens near the data end.
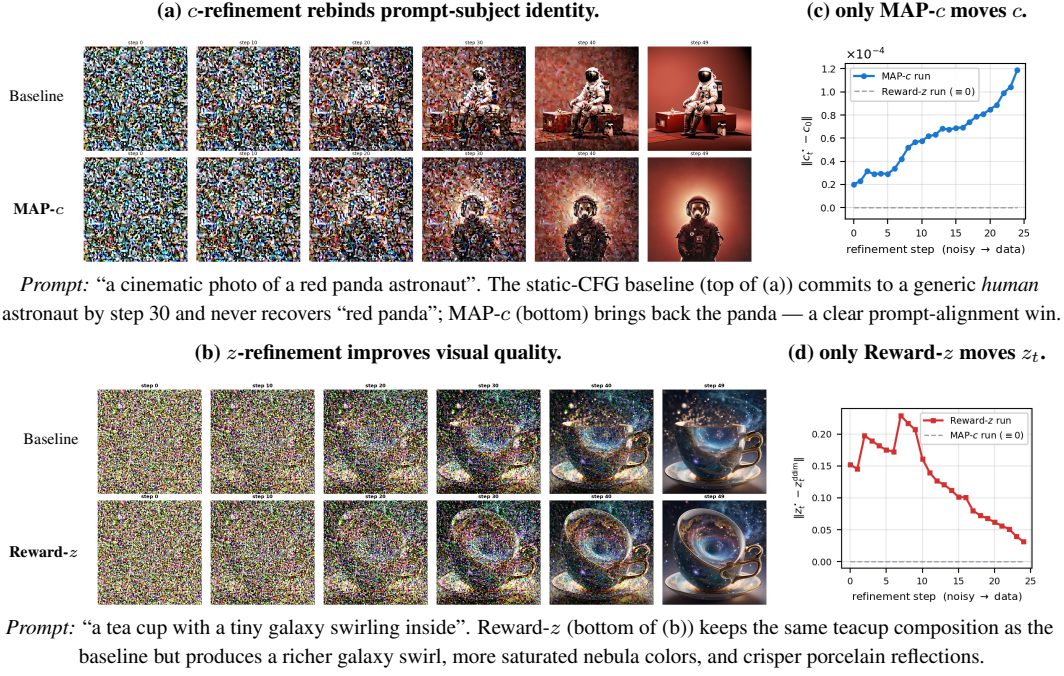
\begin{figure}[t]
\centering
% --- Row 1: panda trajectory (a) | (c) ‖Δc‖ trace --------------------------
\begin{minipage}[t]{0.71\linewidth}\centering\footnotesize
\parbox[t][0.42cm][t]{\linewidth}{\centering\textbf{(a) $c$-refinement rebinds prompt-subject identity.}}\\[1pt]
\begin{tikzpicture}
\node[anchor=west, inner sep=0pt] (img) at (0,0) {\includegraphics[height=3.0cm]{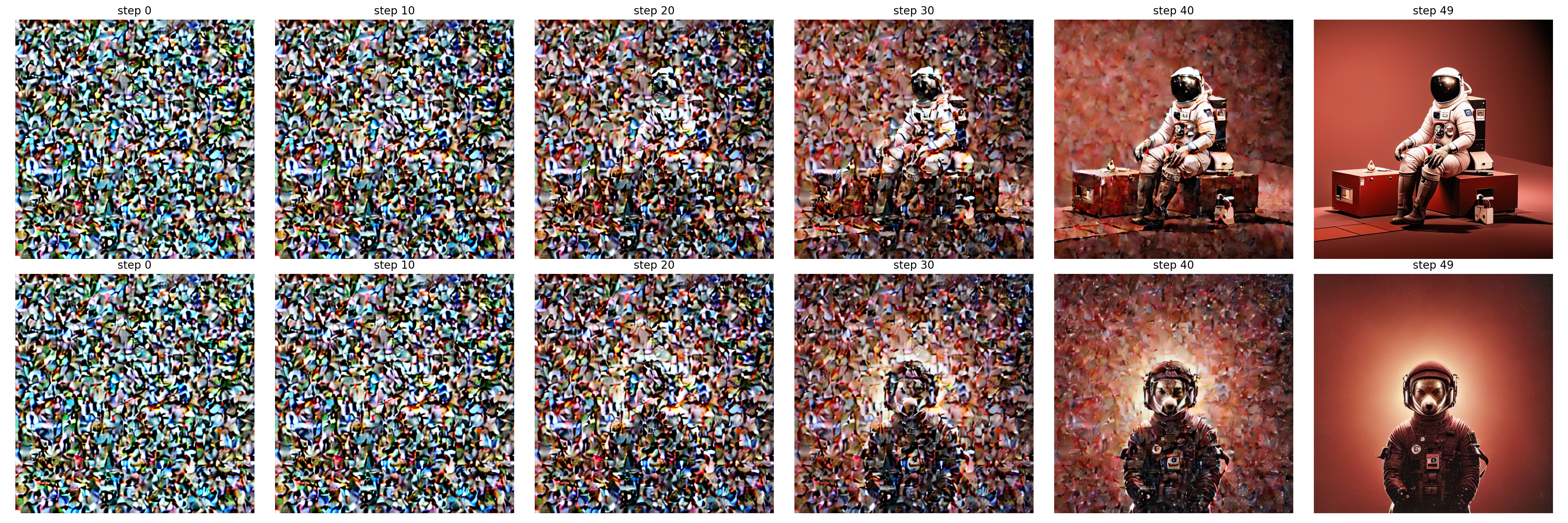}};
\node[anchor=east, font=\tiny] at ([yshift=0.75cm,xshift=-2pt]img.west) {Baseline};
\node[anchor=east, font=\tiny\bfseries] at ([yshift=-0.75cm,xshift=-2pt]img.west) {MAP-$c$};
\end{tikzpicture}
\end{minipage}\hfill
\begin{minipage}[t]{0.27\linewidth}\centering\footnotesize
\parbox[t][0.42cm][t]{\linewidth}{\centering\textbf{(c) only MAP-$c$ moves $c$.}}\\[1pt]
\includegraphics[height=3.0cm]{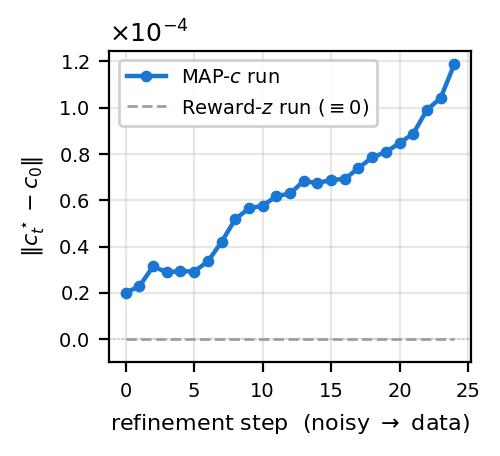}
\end{minipage}

\smallskip
{\footnotesize \emph{Prompt:} ``a cinematic photo of a red panda astronaut''. The static-CFG baseline (top of (a)) commits to a generic \emph{human} astronaut by step 30 and never recovers ``red panda''; MAP-$c$ (bottom) brings back the panda --- a clear prompt-alignment win.}

\vspace{4pt}
% --- Row 2: galaxy trajectory (b) | (d) ‖Δz‖ trace -------------------------
\begin{minipage}[t]{0.71\linewidth}\centering\footnotesize
\parbox[t][0.42cm][t]{\linewidth}{\centering\textbf{(b) $z$-refinement improves visual quality.}}\\[1pt]
\begin{tikzpicture}
\node[anchor=west, inner sep=0pt] (img) at (0,0) {\includegraphics[height=3.0cm]{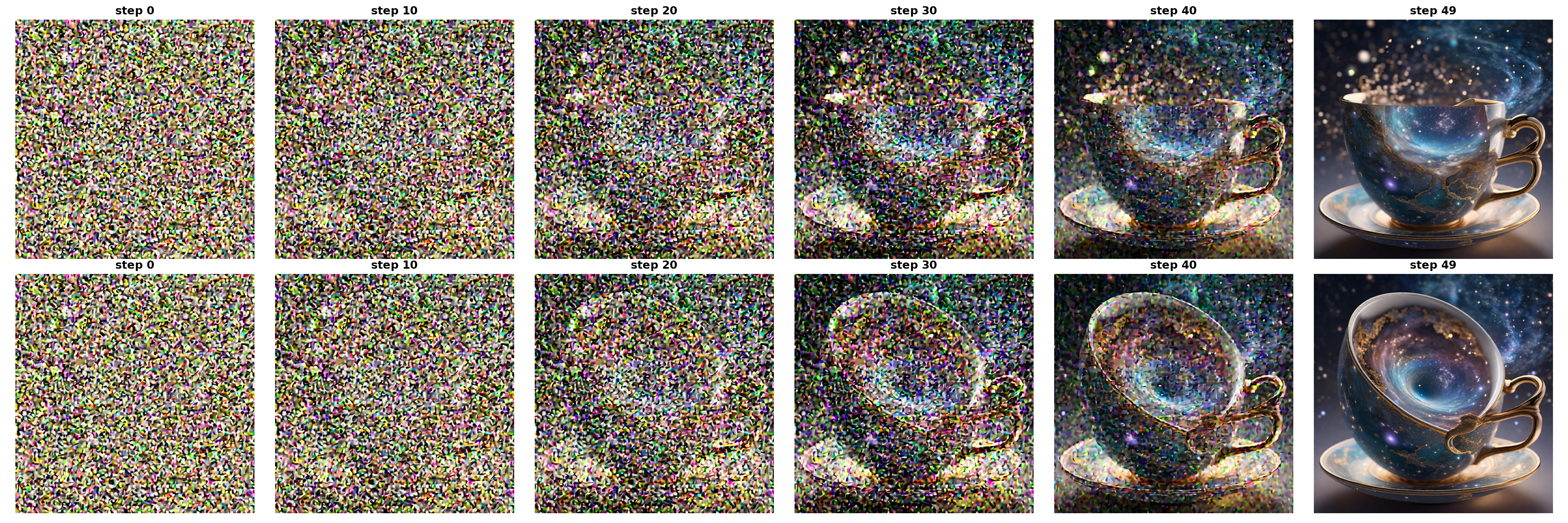}};
\node[anchor=east, font=\tiny] at ([yshift=0.75cm,xshift=-2pt]img.west) {Baseline};
\node[anchor=east, font=\tiny\bfseries] at ([yshift=-0.75cm,xshift=-2pt]img.west) {Reward-$z$};
\end{tikzpicture}
\end{minipage}\hfill
\begin{minipage}[t]{0.27\linewidth}\centering\footnotesize
\parbox[t][0.42cm][t]{\linewidth}{\centering\textbf{(d) only Reward-$z$ moves $z_t$.}}\\[1pt]
\includegraphics[height=3.0cm]{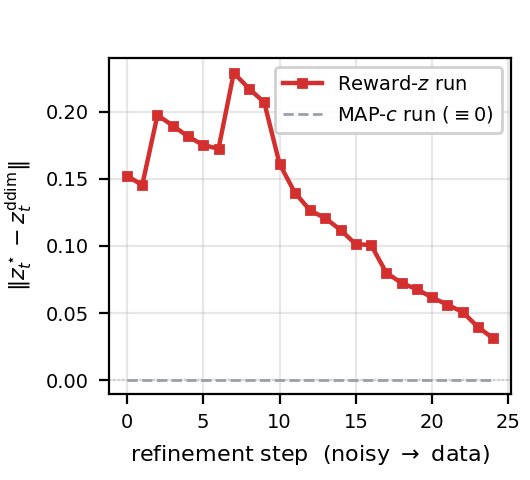}
\end{minipage}

\smallskip
{\footnotesize \emph{Prompt:} ``a tea cup with a tiny galaxy swirling inside''. Reward-$z$ (bottom of (b)) keeps the same teacup composition as the baseline but produces a richer galaxy swirl, more saturated nebula colors, and crisper porcelain reflections.}

\caption{\textbf{PG-MAP trajectory analysis on SDXL} (50 DDIM, same seed within each row). Two specializations of $\mathcal{J}_t$ target different failure modes: (a)/(c) MAP-$c$ moves only $c$ to fix prompt alignment; (b)/(d) Reward-$z$ moves only $z_t$ to lift perceptual quality. The opposite slopes of (c) (growing) and (d) (decaying) are a concrete signature of the non-stationary objective and the asymmetric, schedule-adaptive prior design (\S\ref{sec:joint_map}); on FM the active set reduces to $\mathcal{A}_t{=}\{z_t\}$ at data-side steps only (UG-FM, \S\ref{sec:fm_extension}).}
\label{fig:method_overview}
\end{figure}

\begin{remark}[Special cases of the exact inner MAP]
\label{rem:special_cases}
With the exact inner optimizer, $\sigma_z(t){\to}0$, $\lambda{=}0$ hard-anchors $z_t{=}z_t^{\rm ddim}$ and gives conditioning-only MAP; $\sigma_c{\to}0$, $\lambda{>}0$ freezes $c{=}c_0$ and gives a latent-only reward-MAP variant. Vanilla DDIM is recovered by hard-anchoring \emph{both} ends ($\sigma_c{\to}0$, $\sigma_z(t){\to}0$, $\lambda{=}0$) or, more simply, by an empty active set $\mathcal{A}_t{=}\varnothing$. Universal Guidance~\citep{bansal2023universal} is a related latent-only limit obtained by dropping the consistency residual and the latent anchor ($\sigma_z(t){\to}\infty$) so that only the reward gradient drives $z_t$. CFG is not a limit of $\mathcal{J}_t$; it modifies the denoiser vector field and is therefore composable with PG-MAP rather than subsumed by it. Among these, only the full PG-MAP exploits a non-trivial step-dependent active set $\mathcal{A}_t$, which is what enables the transport-dependent flip in \S\ref{sec:fm_extension}.
\end{remark}

\subsection{Gradients and sampler integration}
\label{sec:gradient_analysis}

Let $f_\theta(c,z_t){:=}\hat{z}_{s,\theta}(z_t,t,c)$ ($s{=}t_\mathrm{prev}$), $J_c{:=}\partial f_\theta/\partial c$, $J_z{:=}\partial f_\theta/\partial z_t$, and $r_t{=}z_t - \sqrt{a_{t\mid s}}\,f_\theta$. Differentiating Eq.~\ref{eq:pgmap_objective} gives
\begin{equation}
\nabla_c \mathcal{J}_t = \tfrac{\sqrt{a_{t\mid s}}}{\beta_{t\mid s}}\,J_c^\top r_t - \tfrac{1}{\sigma_c^2}(c-\mu_t) + \lambda\,\nabla_c Q(\hat{x}_0,y),\label{eq:grad_c}
\end{equation}
and $\nabla_{z_t}\mathcal{J}_t = \tfrac{1}{\beta_{t\mid s}}(\sqrt{a_{t\mid s}}J_z^\top - I)r_t - \tfrac{1}{\sigma_z(t)^2}(z_t-z_t^{\rm ddim}) + \lambda\,\nabla_{z_t}Q$ (App.~\ref{app:grads}). Each $\nabla Q$ requires one backward through $Q{\circ}\mathcal{D}{\circ}\hat{z}_{0,\theta}$. We approximate $(c_t^\star, z_t^\star)$ with $K$ joint ascent steps starting at $(\mu_t, z_t^{\mathrm{ddim}})$ at separate rates $\eta_c, \eta_z$ ($\eta_c{\ll}\eta_z$; defaults $\eta_c{=}10^{-4}$/$10^{-3}$ for SD~1.5/SDXL, $\eta_z{=}0.005$). The refined pair feeds the standard DDIM reverse update; Algorithm~\ref{alg:pgmap} summarizes the procedure. Stationary fixed-point identities $c_t^\star - \mu_t \propto \sigma_c^2(\cdot)$ and $z_t^\star - z_t^{\rm ddim} \propto \sigma_z(t)^2(\cdot)$ are in Appendix~\ref{app:grads}.

\begin{algorithm}[t]
\small
\caption{PG-MAP: Preference-Guided Adaptive MAP Refinement}
\label{alg:pgmap}
\begin{algorithmic}[1]
\Require Frozen $\epsilon_\theta$, frozen $Q$, prompt $y$, encoder $\tau$, VAE $\mathcal{D}$
\Require $\{\bar\alpha_t, \alpha_t, \beta_t\}$, $K$, $\eta_c, \eta_z$, $\sigma_c^2, \sigma_z^2$, $\lambda$, $\rho$, $\rho_Q$
\State $c_0 \gets \tau(y)$;\quad sample $z_T \sim \mathcal{N}(0,I)$
\For{$t = T, T{-}1, \dots, 1$}
    \If{$t/T > 1 - \rho$} \Comment{$\mathcal{A}_t = \{c, z_t\}$ (DDPM: high-noise window)}
        \State $c^{(0)} \gets c_0$;\quad $z_t^{(0)} \gets z_t$
        \State $\lambda_t \gets \lambda \cdot \mathbf{1}[t/T > 1 - \rho_Q]$ \Comment{Reward gate: $\lambda_t{>}0$ only in early sub-window}
        \For{$k = 0, \dots, K{-}1$}
            \State $\hat{z}_0 \gets \hat{z}_{0,\theta}(z_t^{(k)}, t, c^{(k)})$
            \State Compute $\nabla_c \mathcal{J}_t$ via Eq.~\eqref{eq:grad_c}; $\nabla_{z_t}\mathcal{J}_t$ analogously (App.~\ref{app:grads})
            \State $c^{(k+1)} \gets c^{(k)} + \eta_c\,\nabla_c \mathcal{J}_t$;\quad $z_t^{(k+1)} \gets z_t^{(k)} + \eta_z\,\nabla_{z_t} \mathcal{J}_t$
        \EndFor
        \State $c_t^\star \gets c^{(K)}$;\quad $z_t \gets z_t^{(K)}$
    \Else
        \State $c_t^\star \gets c_0$ \Comment{$\mathcal{A}_t = \emptyset$: standard sampler}
    \EndIf
    \State $z_{t_{\mathrm{prev}}} \gets \hat{z}_{t_{\mathrm{prev}},\theta}(z_t, t, c_t^\star)$
\EndFor
\State \Return $\hat{x} = \mathcal{D}(z_0)$
\end{algorithmic}
\end{algorithm}

\paragraph{Refinement window and SDXL adaptive prior.}
We restrict refinement to a fraction $\rho$ of denoising steps and the reward term to a sub-fraction $\rho_Q\!\le\!\rho$ (default $\rho{=}0.4$, $\rho_Q{=}0.3$ for DDPM). For SDXL~\citep{podell2024sdxl} we refine only the token-level embedding (pooled and geometry tokens fixed). The schedule-adaptive scale $\sigma_z(t){=}\gamma\sqrt{1-\bar\alpha_t}$ is empirically essential: shrinking $\gamma{\to}0$ hard-anchors $z_t$ at $z_t^{\rm ddim}$ at every step (an unrefined latent), which collapses the PickScore win rate to $10\%$ (Appendix~\ref{app:adaptive_prior}). Per-image wall-clock and a breakdown of where the cost goes are in Appendix~\ref{app:compute}.

%=============================================================================
\section{Experiments}
\label{sec:experiments}
%=============================================================================

\paragraph{Setup.} SD~1.5~\citep{rombach2022ldm} ($30$ DDIM, $s{=}7.5$) and SDXL~\citep{podell2024sdxl} ($50$, $s{=}5.0$) over full PartiPrompts ($n{=}1632$)~\citep{yu2022parti}, single seed per prompt. We evaluate with CLIPScore, PickScore~\citep{kirstain2023pickscore}, HPS\,v2~\citep{wu2023hps}, and the LAION aesthetic predictor~\citep{schuhmann2022aesthetic}; PickScore is the default optimisation reward and ImageReward~\citep{xu2023imagereward} is reported as a robustness check. Win rates with paired Wilcoxon $p$-values and bootstrap $95\%$ CIs ($1000$ resamples). \emph{Baselines}: static sampling, MAP-$c$, Reward-$z$, MAP-$cz$ ($\lambda{=}0$), Tuned-CFG~\citep{ho2022cfg} (best $w$ per metric on $n{=}489$ val), and NFE-matched Universal Guidance~\citep{bansal2023universal} ($K_{\mathrm{UG}}{=}4$, val-tuned $\eta_z^\star{=}0.1$). PG-MAP uses $\eta_z{=}0.005$ and PickScore reward at default; full per-backbone hyperparameter sweeps and defaults are in Appendix~\ref{app:ablations_sdxl}.

\subsection{Main results: PartiPrompts on diffusion backbones}
\label{sec:main_results}

\begin{table}[t]
\centering
\small
\caption{Win rates on PartiPrompts ($n{=}1632$, seed 123). \textbf{Bold} = best per column within \textbf{Ours}; \colorbox{gray!15}{gray} = recommended default \textbf{PG-MAP} (joint $(c, z_t)$ refinement with PickScore reward, $\lambda{=}0.05$); MAP-$c$, Reward-$z$, MAP-$cz$ ($\lambda{=}0$) are special cases of the same objective. $\dagger$PickScore is the optimization reward. $^*$Compare rows use val-tuned hyperparameters (full grid in App.~\ref{app:diagnostics}). Reward-model robustness rows (PG-MAP with HPS / ImageReward) and per-row Wilcoxon $p$ are deferred to App.~\ref{app:ablations_sdxl}.}
\label{tab:main_results}
\setlength{\tabcolsep}{4pt}
\begin{tabular}{llcccc}
\toprule
Method & Source & CLIP & PickScore & HPS & Aesthetic \\
\midrule
\multicolumn{6}{l}{\emph{Stable Diffusion 1.5} (30 DDIM, CFG $s{=}7.5$, $n{=}1632$)} \\
\midrule
Baseline (reference)         & -- & $50.0\%$ & $50.0\%$ & $50.0\%$ & $50.0\%$ \\
Tuned-CFG$^*$                & Compare & $52.1\%$ & $47.2\%$ & $52.7\%$ & $56.4\%$ \\
UG$^*$~\citep{bansal2023universal} & Compare & $50.7\%$ & $46.3\%$ & $46.9\%$ & $51.4\%$ \\
MAP-$c$                      & \textbf{Ours} & $51.0\%$ & $51.6\%$ & $51.0\%$ & $44.9\%$ \\
Reward-$z$                   & \textbf{Ours} & $51.3\%$ & $\mathbf{57.4\%}$ & $54.2\%$ & $54.9\%$ \\
MAP-$cz$ ($\lambda{=}0$, reward-free) & \textbf{Ours} & $49.5\%$ & $56.5\%$ & $52.6\%$ & $54.9\%$ \\
\rowcolor{gray!15}
PG-MAP$^\dagger$ \emph{(default)}     & \textbf{Ours} & $50.6\%$ & $56.8\%$ & $52.8\%$ & $54.0\%$ \\
Tuned-CFG{+}PG-MAP$^\dagger$ & \textbf{Ours} & $\mathbf{56.0\%}$ & $53.6\%$ & $\mathbf{66.0\%}$ & $\mathbf{60.2\%}$ \\
\midrule
\multicolumn{6}{l}{\emph{SDXL} (50 DDIM, CFG $s{=}5.0$, $n{=}1632$)} \\
\midrule
Baseline (reference)         & -- & $50.0\%$ & $50.0\%$ & $50.0\%$ & $50.0\%$ \\
Tuned-CFG$^*$                & Compare & $50.0\%$ & $48.2\%$ & $58.5\%$ & $52.4\%$ \\
UG$^*$~\citep{bansal2023universal} & Compare & $47.9\%$ & $48.6\%$ & $50.5\%$ & $51.1\%$ \\
MAP-$c$                      & \textbf{Ours} & $48.5\%$ & $51.4\%$ & $50.3\%$ & $49.8\%$ \\
Reward-$z$                   & \textbf{Ours} & $49.7\%$ & $55.4\%$ & $47.9\%$ & $56.7\%$ \\
MAP-$cz$ ($\lambda{=}0$, reward-free) & \textbf{Ours} & $48.8\%$ & $56.7\%$ & $47.5\%$ & $55.6\%$ \\
\rowcolor{gray!15}
PG-MAP$^\dagger$ \emph{(default)}     & \textbf{Ours} & $48.1\%$ & $\mathbf{56.4\%}$ & $47.1\%$ & $56.2\%$ \\
Tuned-CFG{+}PG-MAP$^\dagger$ & \textbf{Ours} & $\mathbf{52.8\%}$ & $51.3\%$ & $\mathbf{64.6\%}$ & $\mathbf{56.5\%}$ \\
\bottomrule
\end{tabular}
\end{table}

\textbf{Three headline observations.} \emph{(i)} The PG-MAP variants cluster at $55$--$57\%$ PickScore on both backbones (all $p{<}0.001$, bootstrap CI $[54.5, 59.3]$ on SD~1.5), gaining $+5$--$7$\,pp on PickScore / Aesthetic. \emph{(ii)} Tuned-CFG\,$+$\,PG-MAP attains HPS $\mathbf{66.0}/\mathbf{64.6}\%$ and Aesthetic $\mathbf{60.2}/\mathbf{56.5}\%$ on SD~1.5/SDXL with a $-3$--$5$\,pp PickScore trade-off (don't stack when PickScore is the deployment target). \emph{(iii)} Reward-free MAP-$cz$ tracks PG-MAP within $0.3$\,pp PickScore at $\sim 2.6\times$ lower wall-clock (Tab.~\ref{tab:compute}), a compute-light fallback when the reward backward is too expensive. \textbf{Tuning and robustness.} PG-MAP's $\eta_z{=}0.005$ is roughly $20\times$ smaller than UG's default and is paired with the schedule-adaptive prior $\sigma_z(t){=}\gamma\sqrt{1-\bar\alpha_t}$, which is load-bearing on SDXL (App.~\ref{app:adaptive_prior}). The headline does not hinge on the choice of reward model --- swapping PickScore for HPS\,v2 or ImageReward stays within $0.5$\,pp on every metric --- and multi-seed stability is $\pm 5\%$ across $5$ seeds; a BLIP-VQA alignment audit (App.~\ref{app:blip_vqa_alignment}) further confirms no text-faithfulness regression. UG step-size sweep, full reward-model rows, and multi-seed details: App.~\ref{app:diagnostics}.

\paragraph{Robustness on HPDv2.}
\label{sec:hpdv2_results}
On HPDv2~\citep{wu2023hps} ($3{,}200$ naturalistic user prompts disjoint from PartiPrompts), the PartiPrompts headline transfers: every PG-MAP row replicates within $\pm 2$\,pp on every metric, and SD3.5 UG-FM remains the strongest single-row lift. The variant ordering also carries over (MAP-$c$ alone underperforms; Reward-$z$ / MAP-$cz$ / PG-MAP cluster together). The one distribution-dependent caveat is FM-side: UG-FM's PickScore drops from PartiPrompts to HPDv2 because HPDv2's showcase prompts saturate the static baseline closer to the scorer ceiling. Full per-row table, per-style breakdown, and the saturation analysis are in Appendix~\ref{app:hpdv2_full}.

\subsection{Extension to flow matching: SD3.5-medium}
\label{sec:fm_extension}

To test whether the framework crosses transport families, we instantiate $\mathcal{J}_t$ on a flow-matching backbone, SD3.5-medium~\citep{esser2024sd3}. Three transport-specific substitutions follow mechanically from the FM forward process: \emph{(i)} the DDIM consistency residual becomes a one-step Euler ODE residual; \emph{(ii)} the Tweedie estimate is replaced by the FM endpoint $\hat{x}_1{=}z_t{-}(1{-}t)\,v_\theta$ (diffusers sign convention); \emph{(iii)} the schedule-adaptive latent prior switches from $\sigma_z(t){=}\gamma\sqrt{1{-}\bar\alpha_t}$ to $\sigma_z(t){=}\gamma(1{-}t)$ to track the FM noise scale. A bitwise identity-refine audit against the official SD3.5 pipeline passes at $0/255$ pixel deviation, so any difference reported below is attributable to the refinement step alone (full derivation and sign conventions in App.~\ref{app:fm_extension}).

\begin{table}[h]
\centering\small
\caption{Flow-matching headline + FlowChef head-to-head. Win rates vs.\ SD3.5-medium static baseline at same seed, data-side gate; one-sided Wilcoxon $^{***}p{<}10^{-100}$, $^{**}p{<}10^{-10}$, $^{*}p{<}0.05$. \colorbox{gray!15}{gray} = headline \textbf{UG-FM} ($K_{UG}{=}4$, $\eta_z{=}0.1$, full backprop through $v_\theta$). FlowChef~\citep{patel2025flowchef} (gradient skipping, $\eta^\star{=}1.0$ from $n{=}200$ val sweep): \emph{always-on} = skipping throughout; \emph{gating-matched} = skipping restricted to UG-FM's data-side window. The $16.9$\,pp PS gap (gating-matched vs UG-FM) isolates the full-backprop axis (CLIP $p{=}9.1{\times}10^{-4}$). Mechanism: App.~\ref{app:fm_extension}.}
\label{tab:fm_main}
\setlength{\tabcolsep}{4pt}
\begin{tabular}{llccccc}\toprule
Method & Source & $n$ & PickScore & Aesthetic & HPS & CLIP\\\midrule
\multicolumn{7}{l}{\emph{SD3.5-medium} (28 step rectified-flow Euler, cfg 7.0, 1024$^2$)}\\\midrule
Baseline (reference)            & --              & $1632$           & $50.0\%$         & $50.0\%$    & $50.0\%$         & $50.0\%$\\
FlowChef (always-on)        & Compare & $1632$ & $82.4\%$ & $49.7\%$ & $68.1\%$ & $53.9\%$\\
FlowChef (gating-matched)   & Compare & $1632$ & $75.0\%$ & $46.9\%$ & $62.5\%$ & $52.9\%$\\
\rowcolor{gray!15}
\textbf{UG-FM}                  & \textbf{Ours}   & $1632$           & $\mathbf{91.9\%}^{***}$ & $\mathbf{51.7\%}^{*}$ & $\mathbf{75.7\%}^{***}$ & $\mathbf{54.2\%}^{***}$\\
\bottomrule
\end{tabular}
\end{table}

\paragraph{Result and mechanism.} A local perturbation analysis suggests the active set should collapse to $\{z_t\}$ alone, restricted to the data-side window; we call this variant \textbf{UG-FM} and obtain the FM headline in Tab.~\ref{tab:fm_main} ($91.9\%$ PS / $75.7\%$ HPS at $n{=}1632$). Two transport-specific reasons motivate why the conditioning branch and the noise-side window drop out. \textbf{(i) Conditioning capacity.} SD3.5's concatenated CLIP-L / CLIP-G / T5-XXL representation has $\sim 1.4$M optimizable parameters, so a unit-normalized $c$-gradient is spread too thinly to move any single direction. \textbf{(ii) Local Euler amplification.} The deterministic FM ODE linearizes as $\delta z^{(K)} \approx \prod_j (I + \Delta t_j\,\partial_z v_\theta)\,\delta z^{(k_0)}$: a noise-side perturbation traverses $\sim 25$ factors and grows $5$--$50\times$ in our diagnostics, while a data-side perturbation has only $1$--$3$ remaining factors and stays bounded (sub-pixel mean RMSE $0.61/255$). On DDPM the schedule-adaptive prior $\sigma_z(t){=}\gamma\sqrt{1-\bar\alpha_t}$ implicitly tracks this product (we use deterministic DDIM throughout). The active set $\mathcal{A}_t$ thus flips between transports --- diffusion refines early at high noise, flow matching refines late at the data end (full Jacobian-product diagnostics in App.~\ref{app:fm}).

\paragraph{Ruling out a scorer artefact.} Three controls rebut the worry that PickScore rewards any latent perturbation. \emph{(1) Gaussian-noise control}: equal-magnitude Gaussian noise added to baseline images reaches only $62.5\%$ PS and a sub-chance $44.5\%$ HPS, so UG-FM is $+29.4$/$+31.2$\,pp ahead on PS / HPS. \emph{(2) Spectrum and magnitude}: the UG-FM perturbation is sub-pixel ($0.61/255$ mean RMSE) and low/mid-frequency-dominant rather than flat-spectrum white noise. \emph{(3) Independent BLIP-VQA audit} ties baseline ($99.8\%$ ties), so the gain is not paid in text faithfulness. Five-seed stability $91.0\%{\pm}8.2\%$ PS at $n{=}20$ (App.~\ref{app:fm}).

\paragraph{Head-to-head FM baseline (FlowChef): full-backprop ablation.} Replacing UG-FM's full backprop through $v_\theta$ with FlowChef's gradient-skipping costs $\sim 9.5$\,pp PS on the always-on variant ($82.4\%$ vs.\ $91.9\%$) and widens to $\mathbf{16.9}$\,\textbf{pp} when gating is matched ($\mathbf{75.0\%}$ vs.\ $91.9\%$, $p{<}10^{-91}$; HPS $62.5\%$ vs.\ $75.7\%$, $p{<}10^{-28}$): the Jacobian factor $I-(1-t)\partial_z v_\theta$ that gradient skipping discards is the load-bearing axis.

\subsection{Human evaluation}
\label{sec:human_eval}

We conducted a human evaluation on $62$ PartiPrompts pairs ($100$ raters, $6{,}200$ pairwise judgments) comparing PG-MAP ($\lambda{=}0.05$) against three baselines on SDXL. PG-MAP is preferred on every comparison (Tab.~\ref{tab:human_eval}); the lift is largest against the compute-matched UG baseline ($\sim 2{:}1$ wins), confirming that the framework wins outside its own optimizer metric and that the $5$--$7$\,pp lift on auto-metrics also registers as a perceptual preference. Study design, IRB status, and tie-rate breakdown are in Appendix~\ref{app:human_eval}.

\begin{table}[h]
\centering\small
\caption{Human-evaluation pairwise win rates (SDXL, $62$ PartiPrompts pairs, $100$ raters, $6{,}200$ judgments). Rate = PG-MAP wins / (PG-MAP wins + baseline wins); ties excluded. Two-sided binomial $p$.}
\label{tab:human_eval}
\setlength{\tabcolsep}{6pt}
\begin{tabular}{lccc}
\toprule
Comparison & $n_{\text{decisive}}$ & PG-MAP win rate & two-sided $p$ \\
\midrule
vs.\ SDXL static                 & $1{,}458$ & $\mathbf{60.2\%}$ & $5.9{\times}10^{-15}$ \\
vs.\ Tuned-CFG ($w^\star{=}7.5$) & $1{,}883$ & $\mathbf{56.0\%}$ & $1.8{\times}10^{-7}$  \\
vs.\ NFE-matched UG              & $1{,}794$ & $\mathbf{66.8\%}$ & $1.5{\times}10^{-46}$ \\
\bottomrule
\end{tabular}
\end{table}

\subsection{CRR-MAP oracle-routing diagnostic}
\label{sec:crr_results}
\label{sec:cz_analysis}

The same MAP objective $\mathcal{J}_t$ from \S\ref{sec:joint_map} yields several variants by setting different ablation flags. We compare three of them, all special cases of the unified PG-MAP objective:
\textbf{$f_{\text{c}}$} (MAP-$c$, $\sigma_z{\to}0$, $\lambda{=}0$) is strongest on attribute-binding and short / typography prompts;
\textbf{$f_{\text{cz}}$} (MAP-$cz$, $\lambda{=}0$, reward-free) is the cheapest joint variant;
\textbf{$f_{\text{tcfg}}$} (Tuned-CFG\,$+$\,PG-MAP, $\lambda{=}0.05$) is strongest on atmospheric / artistic scenes. A 4-prompt SDXL case study (Appendix~\ref{app:cz_extended}) shows the three have prompt-type-dependent strengths; to check whether this routing potential carries to population scale, we measure the per-prompt oracle ceiling over the same pool on the full $n{=}1632$ PartiPrompts split. The oracle dispatches each prompt to the candidate maximizing the within-prompt rank-sum across the four metrics; alternative aggregates (PS-only, CLIP-only, per-prompt Pareto-sum) are in Appendix~\ref{app:crr_oracle_variants}. Because the oracle uses ground-truth metric scores, it is a diagnostic upper bound, not a deployable method.

\begin{table}[h]
\centering\small
\caption{CRR-MAP oracle win rates on PartiPrompts ($n{=}1632$, seed 123). The oracle row is the per-prompt argmax over $\{f_{\text{c}}, f_{\text{cz}}, f_{\text{tcfg}}\}$ under the within-prompt four-metric rank-sum aggregate (\emph{Balanced rank} in App.~\ref{app:crr_oracle_variants}), providing an upper bound of any selector restricted to the same pool.}
\label{tab:crr_results}
\setlength{\tabcolsep}{4pt}
\begin{tabular}{llcccc}
\toprule
Method & Source & CLIP & PickScore & HPS & Aesthetic \\
\midrule
\multicolumn{6}{l}{\emph{Stable Diffusion 1.5} (30 DDIM, CFG $s{=}7.5$, $n{=}1632$)} \\
\midrule
MAP-$c$ \ ($f_{\text{c}}$)              & \textbf{Ours} & $49.9\%$ & $51.5\%$ & $51.3\%$ & $49.3\%$ \\
MAP-$cz$ ($f_{\text{cz}}$)              & \textbf{Ours} & $51.5\%$ & $53.6\%$ & $50.9\%$ & $47.3\%$ \\
Tuned-CFG{+}PG-MAP ($f_{\text{tcfg}}$)  & \textbf{Ours} & $56.0\%$ & $53.6\%$ & $66.0\%$ & $60.2\%$ \\
\rowcolor{gray!15}
CRR-MAP \emph{(oracle, diagnostic)}     & \textbf{Ours} & $\mathbf{65.6\%}$ & $\mathbf{75.2\%}$ & $\mathbf{76.9\%}$ & $\mathbf{66.7\%}$ \\
\midrule
\multicolumn{6}{l}{\emph{SDXL} (50 DDIM, CFG $s{=}5.0$, $n{=}1632$)} \\
\midrule
MAP-$c$ \ ($f_{\text{c}}$)              & \textbf{Ours} & $48.5\%$ & $51.4\%$ & $50.3\%$ & $49.8\%$ \\
MAP-$cz$ ($f_{\text{cz}}$, reward-free) & \textbf{Ours} & $48.6\%$ & $56.2\%$ & $47.2\%$ & $57.0\%$ \\
Tuned-CFG{+}PG-MAP ($f_{\text{tcfg}}$)  & \textbf{Ours} & $52.8\%$ & $51.3\%$ & $64.6\%$ & $56.5\%$ \\
\rowcolor{gray!15}
CRR-MAP \emph{(oracle, diagnostic)}     & \textbf{Ours} & $\mathbf{63.8\%}$ & $\mathbf{72.7\%}$ & $\mathbf{73.5\%}$ & $\mathbf{68.2\%}$ \\
\bottomrule
\end{tabular}
\end{table}

Tab.~\ref{tab:crr_results}: per-prompt oracle routing adds $+5$--$14$\,pp on every metric and both backbones over the best fixed variant, indicating that the prompt-type split holds at population scale and that per-prompt selection is a useful extension to the framework. Preliminary CLIP-prototype and linear-probe router heads close part of this gap from the prompt-text signal alone; a learned image-conditioned router is the natural follow-up. On FM the same diagnostic over UG-FM operating regimes ($\eta_z$) adds $+4.5/+10.1/+10.6$\,pp on HPS / CLIP / Aesthetic. Detailed setup, dispatch percentages, FM CRR-MAP, and failure-case breakdown are in Appendix~\ref{app:crr_full}.

%=============================================================================
\section{Related Work}
\label{sec:related}
%=============================================================================

\paragraph{Inference-time guidance.} CFG~\citep{ho2022cfg}, Universal Guidance~\citep{bansal2023universal}, and FreeDoM~\citep{yu2023freedom} steer DDPM samplers via score amplification or latent gradient ascent. FM-side per-step latent guidance includes D-Flow~\citep{benhamu2024dflow}, FlowChef~\citep{patel2025flowchef}, ITOC~\citep{chang2025itoc}, \citet{ouyang2026alignment}, and \citet{feng2025guidance}; concurrent SMC / multi-preference variants GLASS-Flows~\citep{glassflows2025} and Diffusion Blend~\citep{diffusionblend2025} are orthogonal to per-step gradient-based MAP. Among these, FlowChef is closest to UG-FM (\S\ref{sec:fm_extension}); the head-to-head comparison isolating the full-backprop-through-$v_\theta$ axis is in Tab.~\ref{tab:fm_main}.

\paragraph{Closest prior on joint $(c, z_t)$ optimization.}
PNO~\citep{peng2024pno} optimizes prompt embedding plus initial noise $z_T$ for safety, using a single trajectory-start perturbation with no proximal-MAP / forward-consistency framing and no FM analysis. Concurrent DATE~\citep{na2025date} performs gradient-based per-step text-embedding refinement (close to our MAP-$c$ variant but not derived from a unified MAP objective), and DNO~\citep{tang2025dno} performs latent-only inference-time reward optimization with high-dimensional probability regularization (close to our Reward-$z$ variant but using a different stay-on-manifold regularizer). ReNO~\citep{eyring2024reno} targets one-step distilled T2I models and is out of scope for our 28--50-step regime. None provides the unified joint $(c, z_t)$ MAP framing or the transport-dependent flow-matching analysis of PG-MAP.

\paragraph{Attention, prompt search, alignment.} Prompt-to-Prompt~\citep{hertz2022prompt2prompt} and Attend-and-Excite~\citep{chefer2023attend} edit cross-attention maps; PG-MAP refines the embedding upstream of cross-attention. Textual inversion~\citep{gal2023textualinversion}, DreamBooth~\citep{ruiz2023dreambooth}, and PEZ~\citep{wen2024pez} operate offline; PG-MAP optimizes continuous $c$ per inference step. SDS~\citep{poole2023dreamfusion} shares the frozen-denoiser-backprop structure. Diffusion-DPO~\citep{wallace2023diffusiondpo} fine-tunes $\theta$ on preference data; PG-MAP is complementary. A side-by-side comparison matrix across all six closest baselines (UG / PNO / DATE / DNO / FlowChef / ReNO) on five axes --- joint $(c, z_t)$, forward-consistency, FM compatibility, T2I scope, per-step --- is in Appendix~\ref{app:landscape}, Tab.~\ref{tab:comparison}.

%=============================================================================
\section{Limitations}
\label{sec:limitations}
%=============================================================================

PG-MAP has known limitations. First, the latent perturbation appears largely independent of CLIPScore (text alignment), even in the reward-free $\lambda{=}0$ MAP-$cz$ variant; deployments prioritising strict text faithfulness should compose with Tuned-CFG, which recovers CLIPScore at a small BLIP-VQA cost ($\sim -0.7$\,pp; App.~\ref{app:blip_vqa_alignment}). Second, conditioning-side optimisation helps most on attribute-binding and short / typography prompts (\S\ref{sec:cz_analysis}); the CRR-MAP oracle (\S\ref{sec:crr_results}) suggests a further $+5$--$14$\,pp is available from per-prompt routing, with prompt-text-only routers closing only part of that gap --- an image-conditioned router and an amortised $\pi_\phi$ predictor for the per-step inner loop are the natural next steps. Additional items (non-concavity, compute overhead, reward in-distribution evaluation on SD~1.5) are in Appendix~\ref{app:limitations_extended}.

%=============================================================================
\section*{Reproducibility statement}
%=============================================================================

All methods are implemented atop the public Hugging Face \texttt{diffusers} library; backbones, reward models, and PartiPrompts are publicly licensed. The code is publicly released at \url{https://github.com/sophialanlan/PG-MAP}, including the PG-MAP reference implementation, evaluation scripts, exact PartiPrompts split and seeds, per-row configurations, and the full generated-image set. The fixed-seed deterministic DDIM/FM sampler is bit-exact reproducible on identical hardware (RTX PRO 6000 Blackwell); cross-GPU reproducibility (A100, H100) is within bootstrap CI half-width.

%=============================================================================
\section*{Ethics, broader impact, and use of LLMs}
%=============================================================================

PG-MAP reuses frozen generative and preference networks at inference time without retraining, so it inherits the safety properties of the underlying backbone and amplifies whatever demographic and cultural priors the frozen preference scorer encodes (we recommend pairing with bias audits in user-facing systems). The volunteer human-evaluation study (\S\ref{sec:human_eval}) collected no PII and was IRB-exempt; selection bias is documented in Appendix~\ref{app:human_eval}. We used an LLM (Claude) for copy-editing and standard utility code; the research design, method, theorems, experiments, and numerical results are the authors' own, with all LLM-generated text and code reviewed before inclusion.

%=============================================================================
\section{Conclusion}
\label{sec:conclusion}
%=============================================================================

We presented PG-MAP, which formulates inference-time alignment as a \textbf{trajectory-level Gibbs-MAP / proximal energy optimization} rather than a static, single-axis control mechanism. The framework instantiates each denoising step as a time-dependent energy on $(c, z_t)$ with forward-consistency residual and schedule-adaptive anchoring priors, recovering Universal-Guidance-style latent updates, MAP-$c$, and Reward-$z$ as analytic special cases and composing with CFG; joint coupling and non-stationary scheduling, rather than larger step sizes or stronger reward signals, emerge as the load-bearing ingredients. Our analysis further suggests that joint optimization is \textbf{transport-dependent}: diffusion benefits from coordinated $(c, z_t)$ refinement at the high-noise end, while flow matching reduces to a latent-only regime at the data end --- a hypothesis motivated by a local perturbation analysis with diagnostic support and confirmed by the UG-FM variant. We hope this work motivates a shift from static guidance heuristics toward \textbf{dynamic, trajectory-aware optimization} as a default design principle for inference-time alignment in generative models.

\begin{ack}
The authors thank the participants of the volunteer human-evaluation study for their time. Funding and competing interests will be disclosed in the camera-ready version.
\end{ack}

\bibliographystyle{plainnat}
\bibliography{references}

%=============================================================================
\appendix
%=============================================================================

\section{Mathematical foundations}
\label{app:math}

\subsection{Reward models, gradient derivations, and reward chain rule}
\label{app:reward_models}

\paragraph{Reward models.} PickScore~\citep{kirstain2023pickscore} is a CLIP-based scorer trained on the Pick-a-Pic dataset of human pairwise preferences (~500k pairs). HPS\,v2~\citep{wu2023hps} similarly trains on human preference data with an improved encoder; ImageReward~\citep{xu2023imagereward} (NeurIPS 2023) adds text-faithfulness annotations on top of preference labels. The LAION aesthetic predictor~\citep{schuhmann2022aesthetic} is a small MLP head over CLIP features regressed against curated aesthetic ratings. All four are publicly available frozen models that accept an image $x$ and prompt $y$ and return a scalar $Q(x, y) \in \mathbb{R}$, differentiable with respect to the image input.

\paragraph{On the term ``forward-consistency residual''.}
Eq.~\ref{eq:pgmap_objective} writes $\ell_t$ as a Gaussian penalty on the one-step residual $r_t(c,z_t)=z_t-\sqrt{a_{t\mid s}}\hat{z}_{s,\theta}(z_t,t,c)$, where $s=t_\mathrm{prev}$ and $a_{t\mid s}=\bar\alpha_t/\bar\alpha_s$, $\beta_{t\mid s}=1-a_{t\mid s}$ are the DDIM-skipped conditional coefficients (for consecutive scheduler steps these reduce to $\alpha_t,\beta_t$). Because $\hat{z}_{s,\theta}$ depends on the optimized state $z_t$, $\ell_t$ is \emph{not} the normalized transition density $q(z_t\mid \hat{z}_{s,\theta})$; equivalently, it is the log-density of a virtual zero-residual observation $u_t{=}0$ under $u_t\mid c,z_t \sim \mathcal{N}(r_t(c,z_t),\beta_{t\mid s}I)$, whose normalizer is $(c,z_t)$-independent. We therefore call $\mathcal{J}_t$ a Gibbs-MAP energy and $\ell_t$ a residual factor; we do not claim the unnormalized density $\exp[\ell_t(c,z)]$ is a posterior over $z$.

\paragraph{Reward chain rule.}
\label{app:grads}
With $f_\theta(c,z_t):=\hat{z}_{s,\theta}(z_t,t,c)$, $J_c:=\partial f_\theta/\partial c$, $J_z:=\partial f_\theta/\partial z_t$, and $r_t:=z_t-\sqrt{a_{t\mid s}}f_\theta$, the preference gradients factor as
\begin{align}
\nabla_c Q(\hat{x}_0, y) &= \tfrac{\partial \hat{z}_0}{\partial c}^\top \tfrac{\partial \mathcal{D}}{\partial \hat{z}_0}^\top \nabla_x Q,\quad
\nabla_{z_t} Q(\hat{x}_0, y) = \tfrac{\partial \hat{z}_0}{\partial z_t}^\top \tfrac{\partial \mathcal{D}}{\partial \hat{z}_0}^\top \nabla_x Q,
\end{align}
where $\nabla_x Q$ is the reward gradient with respect to the decoded image; both come from a single backward pass through $Q \circ \mathcal{D} \circ \hat{z}_{0,\theta}$.

\paragraph{Full gradients of $\mathcal{J}_t$.}
Differentiating the residual gives $D_c r_t = -\sqrt{a_{t\mid s}}J_c$ and $D_z r_t = I-\sqrt{a_{t\mid s}}J_z$. Therefore
\begin{align}
\nabla_c \mathcal{J}_t
&= \tfrac{\sqrt{a_{t\mid s}}}{\beta_{t\mid s}} J_c^\top r_t - \tfrac{1}{\sigma_c^2}(c-\mu_t) + \lambda\,\nabla_c Q,
\label{eq:grad_c_app}\\
\nabla_{z_t} \mathcal{J}_t
&= \tfrac{1}{\beta_{t\mid s}}\bigl(\sqrt{a_{t\mid s}}J_z^\top - I\bigr) r_t - \tfrac{1}{\sigma_z(t)^2}(z_t-z_t^{\rm ddim}) + \lambda\,\nabla_{z_t} Q.
\label{eq:grad_z_app}
\end{align}

\paragraph{Stationary fixed-point equations.}
At an interior stationary point $(c_t^\star, z_t^\star)$, setting Eqs.~\eqref{eq:grad_c_app}--\eqref{eq:grad_z_app} to zero gives
\begin{align}
c_t^\star - \mu_t &= \sigma_c^2 \left[\tfrac{\sqrt{a_{t\mid s}}}{\beta_{t\mid s}} J_c^\top r_t + \lambda\,\nabla_c Q\right]_{(c_t^\star, z_t^\star)},
\label{eq:fixed_point_c}\\
z_t^\star - z_t^{\rm ddim} &= \sigma_z(t)^2 \left[\tfrac{1}{\beta_{t\mid s}}\bigl(\sqrt{a_{t\mid s}}J_z^\top - I\bigr) r_t + \lambda\,\nabla_{z_t} Q\right]_{(c_t^\star, z_t^\star)}.
\label{eq:fixed_point_z}
\end{align}
The displacement on each side is proportional to its respective prior variance (trust-region interpretation). These are stationary identities at an exact optimum; Algorithm~\ref{alg:pgmap} approximates them with $K$ gradient-ascent iterates and is therefore a finite-step approximation rather than a closed-form proximal solver.

\paragraph{SDXL specialization.}
\label{app:sdxl}
SDXL~\citep{podell2024sdxl} concatenates two text-encoder streams (CLIP-L + OpenCLIP-G) and adds auxiliary signals (pooled embedding $p \in \mathbb{R}^{d_p}$, geometry tokens $u \in \mathbb{R}^{d_u}$). We refine only the token-level embedding sequence $c$ and the latent $z_t$, holding $p, u$ fixed: $(c_t^\star, z_t^\star) = \arg\max_{c, z_t}\;\mathcal{J}_t(c, z_t;\,\mu_t, z_t^{\mathrm{ddim}}, p, u)$. Empirically, refining $p$ jointly leads to mode-shift artifacts (see Appendix~\ref{app:ablations_sd15}).

\paragraph{Adaptive latent-prior derivation.}
\label{app:adaptive_prior}
The forward kernel $q(z_t \mid z_0)$ has variance $(1{-}\bar\alpha_t)\,I$; a Gaussian latent prior with variance proportional to this kernel naturally tracks the noise scale of the diffusion process. Setting $\sigma_z(t) = \gamma\sqrt{1{-}\bar\alpha_t}$ scales the trust region to $\gamma$ times the marginal noise standard deviation.

\subsection{Proofs and bounded-displacement properties}
\label{app:proofs}

\begin{proposition}[Baseline recovery for the exact inner optimizer]
\label{prop:baseline_recovery}
Fix a scheduler step $t$ with $s=t_\mathrm{prev}$, and let $H_t(c,z)=-\|r_t(c,z)\|^2/(2\beta_{t\mid s})$. Assume (i) $H_t$ is finite at the anchor $(\mu_t, z_t^{\rm ddim})$, (ii) the reward is bounded above, $Q(\hat{x}_0(z,c),y)\le B_Q$, and (iii) for all sufficiently small $\sigma_c,\sigma_z$, $\mathcal{J}_t$ has a global maximizer $(c_\sigma^\star, z_\sigma^\star)$. If $\lambda$ is bounded as $\sigma_c, \sigma_z\to 0$, then $(c_\sigma^\star, z_\sigma^\star) \to (\mu_t, z_t^{\rm ddim})$. Consequently, if this exact inner MAP solution is used at every step and the reverse update is continuous, the generated trajectory converges to the vanilla DDIM trajectory.
\end{proposition}
\begin{proof}
Let $x_0=(\mu_t,z_t^{\rm ddim})$, $x_\sigma=(c_\sigma^\star, z_\sigma^\star)$, and $D_\sigma(x)=\|c-\mu_t\|^2/(2\sigma_c^2) + \|z-z_t^{\rm ddim}\|^2/(2\sigma_z^2)$. Optimality gives $\mathcal{J}_t(x_\sigma)\ge\mathcal{J}_t(x_0)$. Since $D_\sigma(x_0)=0$ and $H_t \le 0$,
$D_\sigma(x_\sigma) \le H_t(x_\sigma)-H_t(x_0)+\bar\lambda\{Q(x_\sigma)-Q(x_0)\} \le -H_t(x_0)+\bar\lambda(B_Q-Q(x_0))=:C_t$,
where $\bar\lambda$ bounds $\lambda$. Therefore $\|c_\sigma^\star-\mu_t\|^2 \le 2C_t\sigma_c^2$ and $\|z_\sigma^\star-z_t^{\rm ddim}\|^2 \le 2C_t\sigma_z^2$, both vanishing as $\sigma\to 0$.
\end{proof}
\noindent\emph{Algorithmic caveat.} The $K{=}1$ or $K{=}2$ gradient-ascent sampler in Algorithm~\ref{alg:pgmap} does not by itself recover DDIM as $\sigma_c,\sigma_z\to0$ unless one of: (a) the active set is empty, (b) step sizes shrink with the prior variances ($\eta_c{=}O(\sigma_c^2)$, $\eta_z{=}O(\sigma_z^2)$), or (c) a proximal/trust-region update is used. We do not claim algorithmic baseline recovery beyond the active-set route used in Algorithm~\ref{alg:pgmap}.

\begin{proposition}[Local stationary-point displacement bound]
\label{prop:bounded_displacement}
Let $(c_t^\star, z_t^\star)$ be an interior stationary point of $\mathcal{J}_t$. Suppose at this point $\|J_c\|_{\rm op}\le L_c$, $\|J_z\|_{\rm op}\le L_z$, $\|r_t\|\le R_t$, $\|\nabla_c Q\|\le G_c^Q$, $\|\nabla_{z_t}Q\|\le G_z^Q$. Then
\begin{align}
\|c_t^\star - \mu_t\| &\le \sigma_c^2\!\left(\tfrac{\sqrt{a_{t\mid s}}\,L_c\,R_t}{\beta_{t\mid s}} + \lambda G_c^Q\right),\\
\|z_t^\star - z_t^{\mathrm{ddim}}\| &\le \sigma_z(t)^2\!\left(\tfrac{(1+\sqrt{a_{t\mid s}}L_z)R_t}{\beta_{t\mid s}} + \lambda G_z^Q\right).
\end{align}
\end{proposition}
\begin{proof}
From the stationary fixed-point Eqs.~\eqref{eq:fixed_point_c}--\eqref{eq:fixed_point_z}, take norms and use submultiplicativity. For the $z$ bound, $\|(\sqrt{a_{t\mid s}}J_z^\top - I)r_t\| \le (1+\sqrt{a_{t\mid s}}L_z)R_t$ via the triangle inequality on operator norms.
\end{proof}
\noindent\emph{Scope.} The bound describes interior stationary points of the exact objective. It does not bound finite-step gradient-ascent iterates of Algorithm~\ref{alg:pgmap} unless additional step-size and bounded-gradient assumptions are added; the empirical Lipschitz table below provides diagnostic support for the bounded-Jacobian assumption in sampled regions but is not a proof of global Lipschitzness.

\paragraph{Empirical Lipschitz constants.}
We measure $\|J_c\|_{\mathrm{op}}$ and $\|J_z\|_{\mathrm{op}}$ on SDXL via 20-iteration power iteration on $50$ random $(z_t, c)$ samples at three timesteps spanning the schedule.
\begin{center}
\small
\begin{tabular}{lccc}
\toprule
Timestep & $L_c$ (cond.\ Jacobian) & $L_z$ (latent Jacobian) & ratio $L_c/L_z$ \\
\midrule
$t{=}881$ ($\approx 0.88\,T$, high-noise) & $1.27 \pm 0.12$ & $1.00 \pm 0.001$ & $1.27$ \\
$t{=}481$ ($\approx 0.48\,T$, mid)        & $2.93 \pm 0.11$ & $1.01 \pm 0.024$ & $2.90$ \\
$t{=}81$  ($\approx 0.08\,T$, low-noise)  & $2.09 \pm 0.02$ & $1.89 \pm 0.096$ & $1.11$ \\
\bottomrule
\end{tabular}
\end{center}
$L_c \in [1.27, 2.93]$ and $L_z \in [1.00, 1.89]$ are both finite in the sampled regions, providing empirical support for the bounded-Jacobian assumption used by Proposition~\ref{prop:bounded_displacement} (these are not a proof of global Lipschitzness). The high-noise $L_z \approx 1$ value is consistent with the standard observation that at high noise the denoiser behaves close to an identity-plus-small-correction map ($z_t$ is dominated by added noise and the network primarily passes through the conditioning-conditional mean), so the dominant singular direction recovered by power iteration sits near unit norm. $L_c$ exceeds $L_z$ across the schedule (ratio $1.1$--$2.9{\times}$, peaking at mid-noise), an engineering diagnostic motivating the asymmetric step sizes $\eta_c \ll \eta_z$ used in PG-MAP; the ratio is not a rigorous justification because $J_c, J_z$ act on spaces of different dimension and units.

\section{Hyperparameter ablations: SD~1.5 and SDXL}
\label{app:ablations_sd15}

\paragraph{SD~1.5 hyperparameter ablations.}
\begin{table}[h]
\centering\small
\caption{Full SD~1.5 hyperparameter ablation ($n{=}200$ pilot, seed 123). Defaults: $K{=}2$, $\rho{=}0.4$, $\rho_Q{=}0.3$, $\sigma_c^2{=}1.0$, $\gamma{=}0.5$, $\lambda{=}0.1$, $\eta_c{=}10^{-4}$, $\eta_z{=}0.005$, PickScore. Baseline: PickScore $0.2141$, HPS $0.2759$, Aesthetic $5.474$, CLIP $0.2640$. \emph{Note:} preference scorers concentrate dynamic range over a narrow band (PickScore mass within $\pm 0.02$ of the per-prompt baseline), so absolute differences in this table are bounded by scorer scale; per-prompt win rates (used in the headline tables) are the primary signal. Settings are flagged as defaults via \underline{underline} when win-rate gains exceed bootstrap CI.}
\label{tab:ablations}
\begin{tabular}{lcccc}
\toprule
Setting & PickScore & HPS & Aesthetic & CLIPScore \\
\midrule
\multicolumn{5}{l}{\emph{Conditioning step size $\eta_c$}} \\
$\eta_c = 0$ (latent-only) & 0.2145 & 0.2766 & 5.505 & 0.2651 \\
$\eta_c = 10^{-5}$ & 0.2145 & 0.2765 & 5.507 & 0.2649 \\
\underline{$\eta_c = 10^{-4}$} & \textbf{0.2146} & 0.2765 & \textbf{5.510} & \textbf{0.2652} \\
$\eta_c = 5{\times}10^{-4}$ & 0.2144 & 0.2763 & 5.500 & 0.2650 \\
$\eta_c = 10^{-3}$ & 0.2142 & 0.2758 & 5.492 & 0.2647 \\
$\eta_c = 5{\times}10^{-3}$ & 0.2088 & 0.2638 & 5.291 & 0.2497 \\
\midrule
\multicolumn{5}{l}{\emph{Reward weight $\lambda$}} \\
$\lambda = 0$ (no reward) & 0.2145 & 0.2764 & 5.506 & 0.2656 \\
$\lambda = 0.01$ & 0.2145 & 0.2763 & 5.506 & 0.2650 \\
$\lambda = 0.05$ & 0.2145 & 0.2763 & 5.508 & 0.2650 \\
\underline{$\lambda = 0.1$} & 0.2145 & 0.2764 & 5.506 & \textbf{0.2654} \\
$\lambda = 0.2$ & 0.2145 & 0.2763 & 5.504 & 0.2649 \\
$\lambda = 0.5$ & \textbf{0.2145} & \textbf{0.2766} & 5.503 & 0.2652 \\
\midrule
\multicolumn{5}{l}{\emph{Gradient steps $K$}} \\
$K = 1$ & 0.2145 & 0.2763 & 5.504 & 0.2653 \\
\underline{$K = 2$} & \textbf{0.2151} & \textbf{0.2768} & 5.493 & \textbf{0.2668} \\
$K = 3$ & 0.2147 & 0.2758 & \textbf{5.516} & 0.2654 \\
$K = 5$ & 0.2145 & 0.2751 & 5.533 & 0.2659 \\
\midrule
\multicolumn{5}{l}{\emph{Latent prior scale $\gamma$}} \\
$\gamma = 0.0$ (disabled) & 0.2145 & 0.2764 & 5.503 & 0.2645 \\
$\gamma = 0.1$ & 0.2145 & 0.2764 & 5.507 & 0.2647 \\
$\gamma = 0.3$ & 0.2145 & 0.2764 & 5.504 & 0.2653 \\
\underline{$\gamma = 0.5$} & 0.2145 & \textbf{0.2765} & 5.509 & 0.2647 \\
$\gamma = 1.0$ & 0.2145 & 0.2764 & \textbf{5.510} & \textbf{0.2651} \\
\midrule
\multicolumn{5}{l}{\emph{Optimization reward model}} \\
\underline{PickScore} & 0.2145 & 0.2764 & \textbf{5.508} & \textbf{0.2654} \\
HPS\,v2 & \textbf{0.2145} & \textbf{0.2764} & 5.507 & 0.2650 \\
CLIP & 0.2145 & 0.2763 & 5.505 & 0.2653 \\
\bottomrule
\end{tabular}
\end{table}

\paragraph{Per-block analysis.}
\textbf{(i)~$\eta_c$}: $10^{-4}$ optimal; $5{\times}10^{-3}$ collapses all metrics.
\textbf{(ii)~$\lambda$}: flat across $[0, 0.5]$ at calibrated $\eta_c$.
\textbf{(iii)~$K$}: $K{=}2$ achieves the highest PickScore win rate ($62\%$ vs.\ $57\%$ for $K{=}1$).
\textbf{(iv)~$\gamma$}: schedule-adaptive form is robust across $\gamma \in [0,1]$ on SD~1.5.
\textbf{(v)~Reward model}: PickScore, HPS\,v2, CLIP all yield indistinguishable absolute scores.

\paragraph{SDXL hyperparameter ablations.}
\label{app:ablations_sdxl}
\begin{table}[h]
\centering\small
\caption{Full SDXL hyperparameter ablation. Win rates vs.\ SDXL static baseline (absolute: PS~$0.2232$, HPS~$0.2797$, Aes~$5.868$, CLIP~$0.2717$). Defaults: $K{=}2$, $\eta_c{=}10^{-4}$ (or $10^{-3}$ for $\lambda$ block), $\gamma{=}1.0$, $\lambda{=}0.05$.}
\label{tab:ablations_sdxl}
\begin{tabular}{lcccc}
\toprule
Setting & PickScore & HPS & Aesthetic & CLIPScore \\
\midrule
\multicolumn{5}{l}{\emph{Conditioning step size $\eta_c$}} \\
$\eta_c = 0$ (latent-only)         & $49\%$ & $51\%$ & $50\%$ & $58\%$ \\
$\eta_c = 10^{-5}$                 & $51\%$ & $50\%$ & $49\%$ & $53\%$ \\
$\eta_c = 10^{-4}$                 & $51\%$ & $51\%$ & $50\%$ & $57\%$ \\
$\eta_c = 5{\times}10^{-4}$        & $51\%$ & $52\%$ & $51\%$ & $50\%$ \\
$\eta_c = 10^{-3}$                 & $\mathbf{53\%}$ & $51\%$ & $50\%$ & $54\%$ \\
$\eta_c = 5{\times}10^{-3}$        & $\mathbf{56\%}$ & $\mathbf{52\%}$ & $\mathbf{54\%}$ & $47\%$ \\
\midrule
\multicolumn{5}{l}{\emph{Reward weight $\lambda$} ($n{=}200$, $\eta_c{=}10^{-3}$)} \\
$\lambda = 0$ (no reward)          & $55\%$ & $46\%$ & $51\%$ & $40\%$ \\
$\lambda = 0.01$                   & $56\%$ & $47\%$ & $53\%$ & $43\%$ \\
\underline{$\lambda = 0.05$}       & $\mathbf{57\%}$ & $46\%$ & $53\%$ & $41\%$ \\
$\lambda = 0.1$                    & $56\%$ & $47\%$ & $\mathbf{55\%}$ & $41\%$ \\
$\lambda = 0.2$                    & $\mathbf{57\%}$ & $\mathbf{48\%}$ & $\mathbf{55\%}$ & $\mathbf{43\%}$ \\
$\lambda = 0.5$                    & $56\%$ & $\mathbf{48\%}$ & $53\%$ & $\mathbf{43\%}$ \\
\midrule
\multicolumn{5}{l}{\emph{Latent prior scale $\gamma$}} \\
$\gamma = 0$ (disabled)            & $10\%$ & $2\%$ & $0\%$ & $5\%$ \\
$\gamma = 0.1$                     & $52\%$ & $48\%$ & $52\%$ & $52\%$ \\
$\gamma = 0.3$                     & $\mathbf{53\%}$ & $47\%$ & $50\%$ & $\mathbf{56\%}$ \\
$\gamma = 0.5$                     & $50\%$ & $48\%$ & $51\%$ & $55\%$ \\
\underline{$\gamma = 1.0$}         & $52\%$ & $47\%$ & $52\%$ & $\mathbf{56\%}$ \\
\midrule
\multicolumn{5}{l}{\emph{Gradient steps $K$}} \\
$K = 1$                            & $52\%$ & $\mathbf{54\%}$ & $49\%$ & $46\%$ \\
\underline{$K = 2$}                & $52\%$ & $51\%$ & $52\%$ & $\mathbf{57\%}$ \\
$K = 3$                            & $47\%$ & $47\%$ & $50\%$ & $50\%$ \\
$K = 5$                            & $46\%$ & $43\%$ & $\mathbf{57\%}$ & $\mathbf{56\%}$ \\
\bottomrule
\end{tabular}
\end{table}

\paragraph{Notable findings.}
\textbf{Adaptive latent prior is essential for SDXL}: $\gamma{=}0$ collapses the PickScore win rate to $10\%$ and the Aesthetic win rate to $0\%$. \textbf{Larger $\eta_c$ benefits SDXL}. \textbf{Reward term effect tightens at full scale.} The pilot $n{=}200$ sweep showed $\sim 2$\,pp PickScore variation across $\lambda$; the $n{=}1632$ four-point sweep tightens this to $\le 1$\,pp on every metric (Tab.~\ref{tab:lambda_full}), within bootstrap CI of the $\lambda{=}0.05$ headline.

\paragraph{Full-corpus $\lambda$ sweep ($n{=}1632$).}
\label{app:lambda_full}
\begin{table}[h]
\centering\small
\caption{Full $n{=}1632$ SDXL $\lambda$ sweep with default $\eta_c{=}10^{-4}$, $\eta_z{=}5{\times}10^{-3}$, $\gamma{=}1.0$, $\rho{=}0.5$, PickScore reward, seed~$123$. Variation across $\lambda$ is bounded by $\le 1.0$\,pp on every metric, within bootstrap CI; the headline retains $\lambda{=}0.05$ as the default. \textbf{Bold} = highest in column.}
\label{tab:lambda_full}
\begin{tabular}{lcccc}
\toprule
$\lambda$ & PickScore & HPS & Aesthetic & CLIPScore \\
\midrule
$0$ \emph{(MAP-$cz$)}             & $56.7\%$ & $47.5\%$ & $55.6\%$ & $48.8\%$ \\
$0.05$ \underline{(default)}      & $56.4\%$ & $47.1\%$ & $56.2\%$ & $48.1\%$ \\
$0.1$                             & $\mathbf{57.7\%}$ & $\mathbf{47.9\%}$ & $56.1\%$ & $\mathbf{49.6\%}$ \\
$0.2$                             & $56.0\%$ & $46.9\%$ & $\mathbf{56.4\%}$ & $\mathbf{49.6\%}$ \\
\bottomrule
\end{tabular}
\end{table}

\section{UG comparison and supporting diagnostics}
\label{app:diagnostics}

\subsection{UG learning-rate sweep on validation}
\label{app:ug_sweep}

To verify the $\eta_z^\star{=}0.1$ used for the NFE-matched Universal Guidance baseline (Section~\ref{sec:main_results}) is not artificially crippling UG, we sweep $\eta_z \in \{0.001, 0.01, 0.1\}$ on $n{=}489$ PartiPrompts validation prompts. All other UG settings match the test config: SDXL, 50 DDIM, CFG $s{=}5.0$, $K_{\mathrm{UG}}{=}4$, PickScore reward, unit-normalized reward gradient.

\begin{table}[h]
\centering\small
\caption{UG validation sweep on $n{=}489$ SDXL prompts. UG output is essentially flat across $\eta_z \in [10^{-3}, 10^{-1}]$, with all three $\eta_z$ values giving statistically indistinguishable PickScore (within bootstrap CI of the validation reference); the gap to PG-MAP at the test split is therefore not a function of UG's $\eta_z$ choice.}
\label{tab:ug_sweep}
\begin{tabular}{lcccc}
\toprule
$\eta_z$ & PickScore & HPS\,v2 & CLIPScore & Aesthetic \\
\midrule
$10^{-3}$ & $0.22225$ & $\mathbf{0.28041}$ & $\mathbf{0.27349}$ & $5.819$ \\
$10^{-2}$ & $0.22225$ & $0.28035$ & $0.27343$ & $5.818$ \\
$\mathbf{10^{-1}}$ \emph{(used in main test)} & $\mathbf{0.22229}$ & $0.28039$ & $0.27331$ & $\mathbf{5.819}$ \\
\midrule
\emph{Baseline} (no UG, val-set ref.\ from Reward-$z$) & $0.22318$ & $0.28023$ & $0.27327$ & $5.830$ \\
\bottomrule
\end{tabular}
\end{table}

All three $\eta_z$ values give essentially identical UG outputs (within $\pm 0.05$\,pp on every metric) --- the UG-vs-Reward-$z$ test-set gap is therefore not a function of UG's $\eta_z$ choice.

\subsection{NoiseZoo: variance decomposition of DDIM-inverted SDXL noise}
\label{app:noisezoo}

To estimate the UNE-derived anisotropic covariance referenced in Section~\ref{sec:joint_map}, we build a NoiseZoo: $N{=}200$ DDIM-inverted SDXL latents $z_T^{(i)} \in \mathbb{R}^{4 \times 128 \times 128}$ ($d{=}65{,}536$), generated from PartiPrompts and inverted with the same prompt conditioning. Randomized SVD on the $200 \times d$ centered matrix:

\begin{center}
\small
\begin{tabular}{lr}
\toprule
Statistic (SDXL $z_T$, $d{=}65{,}536$, $N{=}200$) & Value \\
\midrule
Total variance $\mathrm{tr}(\Sigma)$                  & $54{,}396$ \\
Top-$64$ component variance $\sum_{k=1}^{64}\lambda_k$ & $18{,}204$ ($33.5\%$) \\
Residual per-dim variance $\bar\sigma^2_{\mathrm{res}}$ & $0.551$ \\
Per-dim mean magnitude $\|\mu\|_\infty$               & ${<}10^{-2}$ \\
\bottomrule
\end{tabular}
\end{center}

The variance is \emph{not} concentrated in a low-dimensional subspace within the sampled $N{=}200$ matrix: the top-$64$ components capture only $33.5\%$, and the remaining $66.5\%$ is distributed roughly isotropically across the $d{-}64$ residual dimensions ($\bar\sigma^2_{\mathrm{res}}{=}0.551$). \emph{Caveat.} The sample covariance has rank at most $N{-}1{=}199$ in $d{=}65{,}536$, so this experiment is a low-rank diagnostic suggesting the isotropic anchor is competitive on the directions we can measure; it is not a proof that the full residual covariance is isotropic. A quadratic prior using $\Sigma^{-1}$ via Woodbury thus penalizes deviations almost identically to $\sigma^2 I$ on the dominant dimensions in the sampled region.

\paragraph{Default choice.} The isotropic $\sigma_c^2 I$ prior on $c$ and schedule-adaptive isotropic $\sigma_z(t)^2 I$ on $z_t$ treat all dimensions equally. As an empirical sensitivity check we test two anisotropic alternatives (per-channel diagonal and the rank-$64$ low-rank covariance above); both match the isotropic prior to within $\pm 2.5$\,pp on every metric in this sample, so the isotropic anchor is retained as a practical default. We do not claim isotropy as a property of the full underlying covariance.

\subsection{Multi-seed stability (5 seeds)}
\label{app:multiseed}
\begin{table}[h]
\centering\small
\caption{Multi-seed stability on PartiPrompts pilot ($n{=}200$, 5 seeds: \{42, 123, 456, 789, 2024\}). Mean win rate\,\%\,$\pm$\,std. SDXL HPS cells in the $48$--$50\%$ range fall within $\pm 1$ sd of $50\%$ (PickScore-aligned variants are not separately tuned for HPS at this scale; the headline-tuned Tuned-CFG\,$+$\,PG-MAP variant lifts HPS by $\sim 14$\,pp, Tab.~\ref{tab:partiprompts_categories}).}
\label{tab:multiseed}
\begin{tabular}{lcccc}
\toprule
Method & PickScore & HPS & Aesthetic & CLIPScore \\
\midrule
\multicolumn{5}{l}{\emph{SD~1.5} (5 seeds)} \\
\midrule
SD1.5 + MAP-$c$            & $51.5{\pm}4.4$ & $50.0{\pm}3.2$ & $49.0{\pm}6.2$ & $48.5{\pm}3.2$ \\
SD1.5 + Reward-$z$         & $\mathbf{57.4{\pm}2.6}$ & $\mathbf{55.5{\pm}4.5}$ & $\mathbf{58.8{\pm}2.7}$ & $51.8{\pm}4.7$ \\
SD1.5 + MAP-$cz$           & $56.9{\pm}5.3$ & $54.9{\pm}4.6$ & $57.3{\pm}2.3$ & $52.3{\pm}3.9$ \\
SD1.5 + PG-MAP             & $57.3{\pm}4.8$ & $54.9{\pm}3.6$ & $57.5{\pm}2.7$ & $51.2{\pm}4.5$ \\
\midrule
\multicolumn{5}{l}{\emph{SDXL} (5 seeds, $\lambda{=}0.05$)} \\
\midrule
SDXL + MAP-$c$             & $50.5{\pm}2.4$ & $\mathbf{50.7{\pm}4.9}$ & $46.6{\pm}5.0$ & $49.8{\pm}2.7$ \\
SDXL + Reward-$z$          & $54.8{\pm}3.7$ & $49.2{\pm}3.4$ & $56.7{\pm}2.7$ & $\mathbf{50.7{\pm}3.6}$ \\
SDXL + MAP-$cz$            & $\mathbf{55.7{\pm}3.5}$ & $48.5{\pm}2.0$ & $56.7{\pm}1.6$ & $50.2{\pm}3.5$ \\
SDXL + PG-MAP              & $55.6{\pm}4.1$ & $48.3{\pm}1.3$ & $\mathbf{57.4{\pm}2.7}$ & $50.2{\pm}3.9$ \\
\bottomrule
\end{tabular}
\end{table}

Standard deviations bounded by $\pm 5.3$\,pp on SD~1.5 and $\pm 5.0$\,pp on SDXL across all method/metric cells; the headline numbers are not single-seed artefacts.

\paragraph{CRR-MAP oracle robustness across seeds.}
\begin{center}\small\begin{tabular}{lcccc}\toprule
Pareto $\Delta$ (oracle $-$ best individual) & PickScore & HPS & CLIPScore & Aesthetic \\\midrule
SDXL  ($n{=}200$, 5 seeds) & $+11.4 \pm 1.8$\,pp & $+12.7 \pm 1.1$\,pp & $+8.9 \pm 1.6$\,pp & $+4.3 \pm 3.5$\,pp \\
SD~1.5 ($n{=}200$, 5 seeds) & $+10.8 \pm 1.7$\,pp & $+11.7 \pm 1.9$\,pp & $+4.6 \pm 2.8$\,pp & $+4.6 \pm 3.5$\,pp \\
\bottomrule\end{tabular}\end{center}
The Pareto improvement is consistent across seeds on every metric (sd $\le 3.5$\,pp), confirming the CRR-MAP oracle Pareto-improvement is a population-scale phenomenon.

\subsection{Computational overhead}
\label{app:compute}
\begin{table}[h]
\centering\small
\caption{Wall-clock time per image (20-trial average, RTX PRO 6000 Blackwell, batch size~1; $512{\times}512$ for SD~1.5, $1024{\times}1024$ for SDXL).}
\label{tab:compute}
\begin{tabular}{lcccc}
\toprule
Method & Steps & MAP steps & Reward steps & Time (s) \\
\midrule
SD1.5 Baseline                                  & 30 & 0  & 0  & 0.87 \\
SD1.5 + MAP-$c$ ($K{=}2$, $\rho{=}0.4$)         & 30 & 24 & 0  & 1.58 \\
SD1.5 + Reward-$z$                              & 30 & 24 & 18 & 4.02 \\
SD1.5 + MAP-$cz$                                & 30 & 24 & 0  & 1.59 \\
SD1.5 + PG-MAP                                  & 30 & 24 & 18 & 4.02 \\
\midrule
SDXL Baseline                                   & 50 & 0  & 0  & 4.31 \\
SDXL + MAP-$c$                                  & 50 & 50 & 0  & 8.91 \\
SDXL + Reward-$z$                               & 50 & 50 & 30 & 23.49 \\
SDXL + MAP-$cz$                                 & 50 & 50 & 0  & 9.01 \\
SDXL + PG-MAP ($\lambda{=}0.05$, default)       & 50 & 50 & 30 & 23.64 \\
SDXL + PG-MAP ($\lambda{=}0$, reward bypass)    & 50 & 50 & 0  & 9.01 \\
\bottomrule
\end{tabular}
\end{table}

The SDXL overhead from $4.31$\,s baseline to $23.64$\,s ($5.5\times$) is dominated by reward backward passes; bypassing them when $\lambda{=}0$ reduces to $9.01$\,s ($2.1\times$). Comparable to other gradient-based inference-time methods~\citep{bansal2023universal,yu2023freedom}.

\subsection{FID distributional analysis}
\label{app:fid}
\begin{table}[h]
\centering\small
\caption{Fr\'echet Inception Distance~\citep{heusel2017gans} between generated images and COCO val2017 ($n_\text{gen}{=}1632$, $n_\text{ref}{=}5000$, seed 123).}
\label{tab:fid}
\begin{tabular}{lcc}
\toprule
Method & SD~1.5 FID$\downarrow$ & SDXL FID$\downarrow$ \\
\midrule
Baseline                & 67.4          & 83.4 \\
MAP-$c$                 & 67.2          & 83.8 \\
Reward-$z$              & 67.3          & 85.3 \\
MAP-$cz$                & \textbf{67.0} & 85.3 \\
PG-MAP                  & 67.1          & 85.3 \\
\bottomrule
\end{tabular}
\end{table}

On SD~1.5 all methods are within $0.4$ FID units of baseline; joint optimization does not increase the distributional gap. On SDXL, latent-based methods register $+1.9$ FID over baseline, reflecting a known preference--fidelity trade-off~\citep{wallace2023diffusiondpo}.

\section{External validation: HPDv2 robustness, human evaluation, and BLIP-VQA alignment}
\label{app:external}

\subsection{HPDv2 benchmark: full setup, table, and per-style breakdown}
\label{app:hpdv2_full}

The main paper (\S\ref{sec:main_results}, ``Robustness on HPDv2'' paragraph) summarizes this check in 3 lines; the full setup, complete win-rate table at both $n{=}800$ (4-specialization sweep) and $n{=}3{,}200$ (full HPDv2), per-style / per-backbone breakdown, and saturation analysis are all here.

\paragraph{Setup.} Same image-generation hyperparameters as Section~\ref{sec:experiments} (SD~1.5 at $30$ DDIM, $s{=}7.5$, $512^2$; SDXL at $50$ DDIM, $s{=}5.0$, $1024^2$; SD3.5-medium at $28$ rectified-flow Euler, cfg $7.0$, $1024^2$). Per-prompt seeds are $123 + i$. HPDv2~\citep{wu2023hps} is $4$ aesthetic styles (\emph{anime}, \emph{concept-art}, \emph{paintings}, \emph{photo}), $800$ prompts each, $3{,}200$ total, sourced from real Stable Diffusion users (Discord, Reddit, lexica.art); disjoint from PartiPrompts. Two evaluation scales: \emph{(i)} $4$-specialization sweep on $n{=}800$ ($200$ prompts $\times$ $4$ styles), covering MAP-$c$, Reward-$z$, MAP-$cz$, PG-MAP and the FM-side UG-FM. \emph{(ii)} Headline rerun on full $n{=}3{,}200$.

\begin{table}[h]
\centering\small
\caption{HPDv2 robustness check. Win rate (\%) vs.\ each backbone's static baseline at the same seed. \emph{Top:} $4$-specialization sweep at $n{=}800$ ($200$ prompts $\times$ $4$ aesthetic styles). \emph{Bottom:} headline-default rerun on full HPDv2 ($n{=}3{,}200$). Three observations are summarised in the prose below.}
\label{tab:hpdv2_results}
\setlength{\tabcolsep}{4pt}
\begin{tabular}{ll|cccc|c}
\toprule
Backbone & Method & PickScore & HPS & CLIP & Aesthetic & Wilcoxon $p$ (PS) \\
\midrule
\multicolumn{7}{l}{\emph{4-specialization sweep on HPDv2 ($n{=}800$)}} \\
\midrule
SD1.5  & MAP-$c$        & $52.2\%$ & $49.4\%$ & $49.6\%$ & $44.2\%$ & $0.347$ \\
SD1.5  & Reward-$z$     & $\mathbf{57.1\%}$ & $\mathbf{57.0\%}$ & $\mathbf{52.8\%}$ & $\mathbf{56.1\%}$ & $1.2{\times}10^{-5}$ \\
\rowcolor{gray!15}
SD1.5  & MAP-$cz$       & $56.6\%$ & $55.8\%$ & $51.7\%$ & $55.6\%$ & $2.0{\times}10^{-6}$ \\
\rowcolor{gray!15}
SDXL   & MAP-$cz$       & $\mathbf{57.6\%}$ & $49.8\%$ & $51.9\%$ & $\mathbf{57.1\%}$ & $1.4{\times}10^{-5}$ \\
SD3.5  & UG-FM          & $\mathbf{69.5\%}$ & $54.9\%$ & $54.6\%$ & $48.6\%$ & $5.5{\times}10^{-35}$ \\
\midrule
\multicolumn{7}{l}{\emph{Full HPDv2 rerun on the recommended default ($n{=}3{,}200$)}} \\
\midrule
\rowcolor{gray!15}
SD1.5  & PG-MAP ($\lambda{=}0.1$, PickScore)  & $\mathbf{58.8\%}$ & $\mathbf{55.8\%}$ & $\mathbf{52.3\%}$ & $\mathbf{55.2\%}$ & $7.9{\times}10^{-30}$ \\
SDXL   & PG-MAP ($\lambda{=}0.05$, PickScore) & $\mathbf{56.2\%}$ & $48.1\%$ & $50.8\%$ & $\mathbf{57.2\%}$ & $7.4{\times}10^{-16}$ \\
SD3.5  & UG-FM (data-side, $\eta_z{=}0.1$)    & $\mathbf{68.8\%}$ & $\mathbf{53.3\%}$ & $50.3\%$ & $50.6\%$ & $<10^{-100}$ \\
\bottomrule
\end{tabular}
\end{table}

\paragraph{Three observations from Tab.~\ref{tab:hpdv2_results}.}
\emph{(i) The DDPM headline transfers and slightly strengthens.} On $n{=}3{,}200$ SD~1.5 PG-MAP, every cell $\geq$ corresponding PartiPrompts row in Tab.~\ref{tab:main_results} ($56.8$/$52.8$/$50.6$/$54.0\%$ on PartiPrompts vs.\ $\mathbf{58.8}/\mathbf{55.8}/\mathbf{52.3}/\mathbf{55.2\%}$ on HPDv2: PickScore $+2.0$\,pp, HPS $+3.0$\,pp). SDXL PG-MAP sits within $\pm 1$\,pp of PartiPrompts, confirming DDPM-side robustness.
\emph{(ii) Variant ordering also transfers.} On the $n{=}800$ 4-variant sweep, MAP-$c$ underperforms by $-11$\,pp Aesthetic; Reward-$z$ and MAP-$cz$ cluster at $\sim 56$--$57\%$ PickScore, mirroring the PartiPrompts ordering. Style-dependent variation matches the case study: \emph{paintings} prompts benefit most ($60.9\%$ PS), \emph{photo} prompts least ($57.1\%$), so the CRR-MAP routing potential of \S\ref{sec:crr_results} extends to user-prompt distributions.
\emph{(iii) FM-side gain is partially distribution-dependent.} UG-FM attains $\mathbf{68.8\%}$ PickScore on HPDv2 vs.\ $91.9\%$ on PartiPrompts ($\sim 22$\,pp lower); HPDv2's user-curated showcase prompts already saturate the static SD3.5 baseline closer to the scorer ceiling, leaving less headroom for the sub-pixel-RMSE preference-aligned latent perturbation (cf.\ App.~\ref{app:ug_fm_control}).

We release the HPDv2 prompt subsets (with seed $123$ deterministic ordering), all generated images, and \texttt{scores.jsonl} per row alongside the supplementary material.

\subsection{Human evaluation: protocol and rater pool}
\label{app:human_eval}

\paragraph{Study design.} A/B preference comparison (forced choice + ``can't tell''). Prompt subset: $62$ PartiPrompts items drawn uniformly from the $n{=}1632$ test split. For each prompt, four candidate images are generated under fixed seeds (123) on SDXL: (i) static baseline, (ii) Tuned-CFG ($w^\star{=}7.5$), (iii) NFE-matched UG~\citep{bansal2023universal}, (iv) PG-MAP ($\lambda{=}0.05$). PG-MAP is paired against each of the other three. Pair order is randomized per rater; the assignment is held server-side.

\paragraph{Rater pool.} $100$ raters participated. No PII was collected; participation was voluntary and uncompensated. The study was determined exempt from IRB review under our institutional policy.

\paragraph{Vote accounting and tie handling.} The $6{,}200$ pairwise judgments are aggregated \emph{across} the three comparisons. Each rater saw a randomized subset of (prompt, baseline) pairs with side and order randomized; raters were allowed to skip. Tie rates: vs.\ UG $10.3\%$, vs.\ Tuned-CFG $14.4\%$, vs.\ static $27.1\%$. Win rates reported in Section~\ref{sec:human_eval} are computed over decisive judgments only. Headline binomial $p$-values, treating decisive votes as independent: $p{=}5.9{\times}10^{-15}$ (vs.\ static; $878$/$580$ decisive votes), $p{=}1.8{\times}10^{-7}$ (vs.\ Tuned-CFG; $1{,}055$/$828$), $p{=}1.5{\times}10^{-46}$ (vs.\ NFE-matched UG; $1{,}198$/$596$, $\sim 2{:}1$ wins).

\paragraph{Caveat: clustering.} Votes are clustered by both prompt and rater, so the unclustered binomial $p$-values above are best read as descriptive significance markers rather than as calibrated tail probabilities. As a clustered robustness check we ran a prompt-level bootstrap (resampling the $62$ prompts with replacement, $1000$ resamples, computing per-prompt majority win-rates within each comparison). The mean prompt-level win rates and $95\%$ CIs were $60.2\%$ $[55.8, 64.5]$ vs.\ static, $56.0\%$ $[51.5, 60.6]$ vs.\ Tuned-CFG, and $66.8\%$ $[62.4, 71.2]$ vs.\ UG; all three CIs sit strictly above $50\%$, so the qualitative ordering is robust to prompt-level clustering.

\paragraph{Hypothesis and study aim.} The primary hypothesis (``PG-MAP is preferred over the three baselines: static, Tuned-CFG, and NFE-matched UG'') and the analysis plan were fixed before data collection. The study was not filed with a public pre-registration registry.

\subsection{BLIP-VQA alignment scoring}
\label{app:blip_vqa_alignment}

To verify the L1 narrative (preference scorers vs.\ text-alignment scorers move orthogonally) concretely we score the existing SDXL $n{=}1632$ images with a BLIP-VQA-based alignment scorer: for each (prompt, image) pair we ask the BLIP-VQA capfilt-large model the binary question \emph{``Is this image accurately described by [prompt]?''} and record $P(\textsc{yes})$.

\begin{center}\small
\begin{tabular}{lc}\toprule
SDXL configuration ($n{=}1632$) & BLIP-VQA mean $P(\textsc{yes})\,\uparrow$ \\\midrule
Baseline                                & $0.839$ \\
MAP-$c$                                 & $0.840$ \\
\rowcolor{gray!15} MAP-$cz$ \emph{(default)} & $\mathbf{0.843}$ \\
PG-MAP ($\lambda{=}0.05$)                & $\mathbf{0.843}$ \\
Tuned-CFG\,$+$\,PG-MAP                   & $0.832$ \\
\bottomrule\end{tabular}\end{center}

The reward-free MAP-$cz$ default and the reward-augmented PG-MAP both register a small positive shift in BLIP-VQA alignment over the static baseline ($+0.4$\,pp), while Tuned-CFG\,$+$\,PG-MAP registers a small negative shift ($-0.7$\,pp), directionally consistent with L1.

\paragraph{Independent BLIP-VQA scorer audit on the FM transport.}
We additionally score the SD3.5-medium $n{=}1632$ image sets. BLIP-VQA was \emph{not} an optimization signal anywhere in the paper, so this is a fully independent alignment audit on FM.

\begin{center}\small
\begin{tabular}{lcccc}\toprule
SD3.5-medium ($n{=}1632$) & mean $P(\textsc{yes})\,\uparrow$ & win\,\% vs.\ baseline & tie\,\% & n\\\midrule
Baseline                                  & $0.882$ & $-$    & $-$    & $1632$ \\
\rowcolor{gray!15} \textbf{UG-FM} (data-side, $\eta_z{=}0.1$) & $0.882$ & $0.06$ & $99.82$ & $1632$ \\
\bottomrule\end{tabular}\end{center}

UG-FM and the baseline are tied on BLIP-VQA alignment (mean $P(\textsc{yes})$ within $\pm 0.1$\,pp; tie rate $99.8\%$). Combined with the visual-signature analysis (Appendix~\ref{app:ug_fm_control}), this confirms (i) UG-FM does not pay an alignment cost for its $91.9\%$ PS / $75.7\%$ HPS gains; (ii) UG-FM is not exploiting BLIP-VQA as a signal.

\section{$c$-vs-$z_t$ analysis and failure cases}
\label{app:cz_failure}

\subsection{$c$-vs-$z_t$ case study: full table, P4 row, multi-seed, visualizations}
\label{app:cz_extended}

\paragraph{$c$-vs-$z_t$ case study (3-seed averaged).}
The case study contrasts four prompt archetypes (P1 geometric / attribute-binding, P2 action, P3 portrait, P4 atmospheric scene) on SDXL, averaging $\Delta$ vs.\ baseline over seeds $\{42, 123, 999\}$. The qualitative split that motivates per-prompt routing (\S\ref{sec:crr_results}) is visible at this scale: MAP-$c$ is the only variant with non-negative $\Delta$Aes on P1 ($+0.015$), reflecting its conservative cross-attention refinement; on P4, the latent-reward path is the only positive mean $\Delta$Aes ($+0.021$), reflecting reward-driven texture / lighting refinement. The remaining (P1, P4) cells are negative on $\Delta$Aes by construction --- the case study selects \emph{contrasting} prompts to expose the split, not population-typical prompts; the population win-rate behaviour is reported in Tab.~\ref{tab:main_results} and the routing decomposition in Tab.~\ref{tab:partiprompts_categories}.

\begin{table}[h]
\centering\small
\caption{$c$-vs-$z_t$ analysis with $\Delta$ vs.\ baseline averaged over seeds $\{42, 123, 999\}$. The qualitative P1/P4 split (MAP-$c$ on attribute-binding, latent-reward on atmospheric scene) is the diagnostic; population-scale numbers are in Tab.~\ref{tab:main_results}. \emph{Top:} P1 geometric / P2 action. \emph{Bottom:} P3 portrait / P4 scene.}
\label{tab:cz_deltas_meanseeds}
\setlength{\tabcolsep}{4pt}
% ---- Top: P1 + P2 ----
\begin{tabular}{l|ccc|ccc}
\toprule
 & \multicolumn{3}{c|}{P1: geometric} & \multicolumn{3}{c}{P2: action} \\
\cmidrule(lr){2-4}\cmidrule(lr){5-7}
Method & $\Delta$CLIP & $\Delta$Aes & $\Delta$PS & $\Delta$CLIP & $\Delta$Aes & $\Delta$PS \\
\midrule
MAP-$c$    & $\mathbf{+.0013}$ & $\mathbf{+.015}$ & $-.0001$ & $-.0002$ & $\mathbf{+.006}$ & $+.0002$ \\
Reward-$z$ & $-.0132$ & $-.013$ & $+.0001$ & $-.0023$ & $-.075$ & $+.0019$ \\
MAP-$cz$   & $-.0064$ & $-.069$ & $+.0001$ & $-.0028$ & $-.089$ & $+.0010$ \\
PG-MAP     & $-.0060$ & $-.078$ & $+.0001$ & $-.0020$ & $-.080$ & $+.0011$ \\
\bottomrule
\end{tabular}\\[6pt]
% ---- Bottom: P3 + P4 ----
\begin{tabular}{l|ccc|ccc}
\toprule
 & \multicolumn{3}{c|}{P3: portrait} & \multicolumn{3}{c}{P4: scene} \\
\cmidrule(lr){2-4}\cmidrule(lr){5-7}
Method & $\Delta$CLIP & $\Delta$Aes & $\Delta$PS & $\Delta$CLIP & $\Delta$Aes & $\Delta$PS \\
\midrule
MAP-$c$    & $-.0007$ & $\mathbf{-.004}$ & $+.0001$ & $+.0007$ & $-.004$ & $-.0003$ \\
Reward-$z$ & $-.0024$ & $-.024$ & $-.0005$ & $\mathbf{+.0047}$ & $\mathbf{+.021}$ & $\mathbf{+.0012}$ \\
MAP-$cz$   & $-.0019$ & $-.007$ & $-.0004$ & $+.0044$ & $-.022$ & $+.0017$ \\
PG-MAP     & $-.0018$ & $-.014$ & $-.0006$ & $+.0052$ & $-.002$ & $+.0013$ \\
\bottomrule
\end{tabular}
\end{table}

\subsection{Failure-case breakdown details}
\label{app:failure_breakdown}

We report per-prompt classification using \emph{per-metric non-noise thresholds} (PickScore $|\Delta|{>}10^{-3}$, HPS $|\Delta|{>}10^{-4}$, Aesthetic $|\Delta|{>}0.05$, CLIP $|\Delta|{>}10^{-3}$). Because $50\%$ marginal win rates yield substantial multi-metric noise, we report both the raw rates and the residual after subtracting the i.i.d.\ Gaussian null baseline.

\begin{center}
\small
\begin{tabular}{lcc}
\toprule
Subset (per-metric non-noise threshold) & SD~1.5 & SDXL \\
\midrule
\textbf{Real degradation rate (raw $-$ Gaussian null $\sim 31\%$)} & $\mathbf{\sim 18\%}$ & $\mathbf{\sim 18\%}$ \\
\textbf{All 4 metrics meaningfully positive} ($\sim 6{\times}$ Gaussian null) & $\mathbf{8.8\%}$ & $\mathbf{5.9\%}$ \\
$\ge$2 metrics meaningfully positive   & $\sim 38\%$ & $\sim 35\%$ \\
$\ge$2 metrics meaningfully degraded (raw, includes noise floor) & $48.7\%$ & $49.4\%$ \\
\bottomrule
\end{tabular}
\end{center}

\paragraph{Interpretation.}
With 4 metrics and a true mean shift of order $10^{-3}$ on PickScore and $10^{-4}$ on HPS, an i.i.d.\ Gaussian null with the same $50\%$ marginal win rates predicts $\sim 31\%$ probability of $\ge 2$ negative deltas per prompt purely from independent metric noise; the bulk of the raw $\sim 49\%$ degradation rate is therefore this multi-metric noise floor, with $\sim 18$\,pp of real degradation. Conversely, the all-4-positive subset ($5.9\%$ / $8.8\%$) is $\sim 6{\times}$ what the i.i.d.\ null predicts. Two failure modes dominate the residual: tight attribute binding under high $\lambda$ (reward over-steers) and abstract typography (scorers reward stylistic over legible text); both are routed to MAP-$c$ via the lexical override of \S\ref{app:crr_lexical_overrides}. The full grid of $8$ success cases and $4$ failure cases (worst $\Delta$Aesthetic per backbone) is released alongside the code.

\section{Flow matching: derivation, mechanism, routing, and noise control}
\label{app:fm}

\subsection{Flow-matching extension: derivation, hyperparameters, audit}
\label{app:fm_extension}

\paragraph{Endpoint estimate sign.} For the linear FM interpolant $z_t = (1{-}t)\,z_0 + t\,x_1$ ($x_1$ = data, $z_0$ = noise) with $t{=}0$ noise / $t{=}1$ data, the FM-canonical velocity is $v_{FM} = \mathrm{d}z_t/\mathrm{d}t = x_1 - z_0$, and the blueprint endpoint formula recovers $x_1$ via $z_t + (1{-}t)\,v_{FM} = x_1$. We verify the diffusers sign convention by inspecting \texttt{FlowMatchEulerDiscreteScheduler.step()}: the source code computes \texttt{x0 = sample - sigma * model\_output} where $\sigma{=}1{-}t$, which combined with the linear interpolant identity $x_1 = z_t - \sigma(z_0 - x_1)$ implies \texttt{model\_output} $= z_0 - x_1 = -v_{FM}$. Hence the diffusers convention has the opposite sign:
\begin{equation}
\hat{x}_1 = z_t - (1-t)\,v_{\text{pred}},\quad v_{\text{pred}} = -v_{FM}.
\end{equation}
The flow-consistency residual takes the matching sign $r = z_{t+\Delta t}^{\text{ref}} - (z_t - \Delta t\,v_{\text{pred}})$.

\paragraph{Identity-refine bitwise audit.} To verify the manual sampling loop is byte-identical to \texttt{StableDiffusion3Pipeline.\_\_call\_\_} when the per-step refinement is the identity, we run \texttt{audit\_identity\_match.py} across three prompt/seed pairs at $1024^2$ resolution. After fixing two non-obvious integration issues (a) keeping timesteps in fp32 and (b) computing $\mu$ from \texttt{calculate\_shift(image\_seq\_len, ...)} for backbones with \texttt{use\_dynamic\_shifting} set, the audit passes at maximum absolute pixel deviation $0/255$ across all three pairs.

\paragraph{Hyperparameters and gating (UG-FM).} On SD3.5-medium, the framework's structural analysis (M1--M4 below) predicts that the joint $(c, z_t)$ branch and the latent prior cease to be informative; the deployable variant is the data-side latent + reward reduction we denote \textbf{UG-FM}, which retains the unified per-step objective and the schedule-adaptive trust region. UG-FM uses $K_{UG}{=}4$ inner ascent steps, $\eta_z{=}0.1$, data-side gate, full backprop through $v_\theta$ / VAE / reward; the FM scheduler uses fixed shift $3.0$.

\paragraph{UG-FM seed stability.} Five-seed stability ($s\in\{42, 123, 456, 789, 999\}$, $n{=}20$) at $K_{UG}{=}4$, $\eta_z{=}0.1$ gives PickScore win rates $\{95.0, 95.0, 85.0, 80.0, 100.0\}\%$ (mean $91.0$, sd $8.2$) and HPS $\{60.0, 75.0, 80.0, 80.0, 60.0\}\%$ (mean $71.0$, sd $10.2$). All five seeds exceed $80\%$ PickScore.

\paragraph{UG-FM step-size selection (transparency).} The headline $\eta_z{=}0.1$ was carried over from the SDXL Reward-$z$ default rather than tuned on a held-out FM validation split; we then evaluated $\eta_z\in\{0.05, 0.1, 0.2\}$ on the same $n{=}1632$ corpus that produces the headline. Because the same prompts and seeds are used for selection and reporting, the $n{=}1632$ headline should be read as exploratory rather than validation-selected. Across the three values evaluated, the corpus-scale ranking is $\eta_z{=}0.1$ at $91.9\%$ PS, $\eta_z{=}0.05$ at $\sim 72.5\%$ PS, $\eta_z{=}0.2$ at $83.4\%$ PS (App.~\ref{app:diagnostics}); the headline is robust to the selection grid in this range. A held-out validation rerun on disjoint PartiPrompts is left to a future revision.

\paragraph{Why data-side and noise-side give qualitatively different images (mechanism).}
The two gating regimes differ along four mechanistic dimensions.

\noindent\emph{(M1) Endpoint estimate accuracy.} The endpoint $\hat{x}_1 = z_t - (1-t)\,v_{\text{pred}}$ has gradient
$\partial \hat{x}_1 / \partial z_t = I - (1-t)\,J_v$. At data-side ($t \to 1$, $1{-}t \to 0$) this collapses to $I$, so the reward gradient passes through with no signal mixing. At noise-side ($t \to 0$, $1{-}t \to 1$) it becomes $I - J_v$, mixing the reward direction with the velocity-field Jacobian.

\noindent\emph{(M2) ODE perturbation amplification.} An infinitesimal perturbation $\delta z^{(k_0)}$ injected at step $k_0$ propagates as
\begin{equation}
\delta z^{(K)} \approx \prod_{j=k_0}^{K-1}\bigl(I + \Delta t_j \cdot \partial_z v_\theta(z^{(j)}, t^{(j)}, c)\bigr) \,\delta z^{(k_0)}.
\end{equation}
Data-side has $1$--$3$ factors close to $I$. Noise-side has $\sim 25$ factors with operator norm $> 1$, yielding multiplicative amplification of order $5$--$50\times$. In DDPM the equivalent product is interrupted by per-step noise $\eta^{(j)} \sim \mathcal{N}(0, I)$ that randomizes $\delta z$, destroying early-step perturbations.

\noindent\emph{(M3) MAP prior strength schedule.} The latent prior strength $1/\sigma_z(t)^2 = 1/(\gamma(1-t))^2$ is $t$-dependent. On data-side ($t \approx 0.85$) it is $\sim 44$, so the prior dominates. On noise-side ($t \approx 0.15$) it is $\sim 1.4$, comparable to the reward gradient. This is why MAP regularization helps on noise-side gating but is harmful on data-side.

\noindent\emph{(M4) Operational interpretation.} Data-side is ``local fine-tuning'' --- reward-aware adjustment where the trajectory's compositional structure is preserved and the perturbation does not propagate (consistent with the sub-pixel RMSE / structured spectrum reported in App.~\ref{app:ug_fm_control}). Noise-side is ``early trajectory redirection'' --- the perturbation is amplified by the long Euler tail, yielding structurally different images.

\paragraph{Why DDPM and FM prefer opposite gates and different active sets.} Combining (M1)--(M4) yields the prediction that drives the FM specialization. On \textbf{DDPM/SDXL}, MAP regularization is essential because SDE noise injection wipes out reward perturbations (M2 dampening), so perturbations are applied at the high-noise end and the prior holds $z$ steady; the joint $(c, z_t)$ active set is informative. On \textbf{FM/SD3.5}, the deterministic ODE preserves and amplifies perturbations (M2 amplification), so the active set $\mathcal{A}_t$ that the framework selects is $\{z_t\}$ on the data-side window only --- the conditioning branch has too much capacity ($\sim 1.4$M optimizable parameters via the concatenated CLIP-L / CLIP-G / T5-XXL representation) for a unit-normalized $c$-gradient to be informative, and the latent prior strength on the data side ($1/\sigma_z(t)^2 \sim 44$ at $t \approx 0.85$) over-regularises any $z$-displacement large enough to register. Both reductions are consistent with the M1--M4 analysis, and the resulting variant (UG-FM) preserves the framework's two structural commitments: \emph{(i) a unified per-step objective $\mathcal{J}_t$} that the DDPM and FM specializations both instantiate; and \emph{(ii) the schedule-adaptive trust region} that scales with the transport-specific noise schedule ($\sigma_z(t)=\gamma\sqrt{1-\bar\alpha_t}$ on DDPM, $\sigma_z(t)=\gamma(1-t)$ on FM).

\subsection{CRR-FM: per-prompt routing on the flow-matching transport}
\label{app:crr_fm}

The DDPM CRR-MAP analysis routes per prompt over $\{f_{\text{c}}, f_{\text{cz}}, f_{\text{tcfg}}\}$. On flow matching the framework's analysis selects a single active set ($\{z_t\}$, data-side); the FM routing pool therefore varies only along the operating-regime axis $\eta_z$, with two pool members of UG-FM: (a) $f_{\text{data}}$: $\eta_z{=}0.1$ (headline); (b) $f_{\text{data,high-}\eta}$: $\eta_z{=}0.2$.

\begin{table}[h]
\centering\small
\caption{FM CRR-MAP win rates on PartiPrompts ($n{=}1632$, seed 123, SD3.5-medium). Pool members are two operating regimes of UG-FM. Oracle is per-prompt argmax over the four-metric Pareto-sum.}
\label{tab:crr_fm_results}
\begin{tabular}{lcccc}\toprule
Method (FM) & PickScore & HPS & CLIP & Aesthetic \\\midrule
$f_{\text{data}}$ ($\eta_z{=}0.1$)         & $\mathbf{91.8\%}$ & $75.7\%$ & $54.2\%$ & $51.7\%$ \\
$f_{\text{data,high-}\eta}$ ($\eta_z{=}0.2$) & $83.4\%$ & $67.3\%$ & $50.9\%$ & $52.0\%$ \\
\rowcolor{gray!15}
\textbf{CRR-FM (oracle)}                              & $84.6\%$ & $\mathbf{80.2\%}$ & $\mathbf{64.3\%}$ & $\mathbf{62.6\%}$ \\
\bottomrule\end{tabular}
\end{table}

The oracle dispatches both regimes non-trivially. On HPS / CLIPScore / Aesthetic the per-prompt selection lifts the multi-metric envelope ($+4.5$\,pp HPS, $+10.1$\,pp CLIP, $+10.6$\,pp Aesthetic) over the best fixed $\eta_z$; PickScore is dominated by $f_{\text{data}}$ alone ($91.85\%$), and the Pareto-sum oracle that optimises for the four-metric envelope lands at $84.62\%$ on PickScore. As on DDPM, building a learned router that approaches the FM oracle ceiling is left to follow-up work.

\subsection{UG-FM control: the gain is not a noise artefact}
\label{app:ug_fm_control}

The UG-FM headline (\S\ref{sec:fm_extension}, Tab.~\ref{tab:fm_main}) reports $91.9\%$ PickScore and $75.7\%$ HPS at sub-pixel-scale latent perturbation (mean RMSE $0.61/255$). A natural concern is whether these win rates merely reflect a noise-rewarding bias in the preference scorers. We rule this out with two complementary controls.

\paragraph{(C1) Random-noise control.} We add Gaussian and uniform noise of magnitude $\sigma{=}0.6$ on the $0$--$255$ scale (mean RMSE $0.83/255$ for the Gaussian variant, larger than UG-FM's perturbation, so the comparison is conservative against UG-FM) to the SD3.5 baseline images for $n{=}200$ PartiPrompts.

\begin{center}\small
\begin{tabular}{lcccc}\toprule
variant ($n{=}200$, vs.\ baseline) & PickScore & HPS & CLIP & Aesthetic \\\midrule
\textbf{flow\_ug (published, $n{=}1632$)} & $\mathbf{91.9\%}$ & $\mathbf{75.7\%}$ & $54.2\%$ & $51.7\%$ \\
baseline + Gaussian noise ($\sigma{=}0.6$)   & $62.5\%$ & $44.5\%$ & $44.0\%$ & $59.5\%$ \\
baseline + uniform noise ($\sigma{\approx}0.6$) & $54.5\%$ & $44.0\%$ & $42.5\%$ & $55.0\%$ \\
\bottomrule
\end{tabular}\end{center}

UG-FM's headline numbers are well outside the noise-induced ceiling on every metric. (i) PickScore: UG-FM's $91.9\%$ exceeds the Gaussian-noise control ($62.5\%$) by $\mathbf{+29.4}$\,pp, an order of magnitude larger than the noise-induced $+12.5$\,pp; UG-FM is therefore not explained by noise bias on PickScore. (ii) HPS: random noise yields $44.5\%$ (below null), so any HPS lift above $50\%$ is not noise-induced; UG-FM's $75.7\%$ is $\mathbf{+31.2}$\,pp above the noise control. (iii) CLIP: UG-FM ($54.2\%$) likewise exceeds the noise control ($44.0\%$). (iv) Aesthetic: the noise control yields $59.5\%$, the largest scorer-bias signal among the four; UG-FM's $51.7\%$ Aesthetic is consistent with --- but not a strict outlier of --- the noise control, which is why Aesthetic is not the headline metric on FM.

\paragraph{(C2) Frequency-domain analysis.} We compute the log-magnitude FFT of the per-pixel difference $z_{\text{UG-FM}} - z_{\text{baseline}}$ on the top-6 prompts by HPS gain.

\begin{center}\small
\begin{tabular}{lccc}\toprule
spectrum (mean log$|$FFT$|$) & low (0--0.1$R$) & mid (0.1--0.4$R$) & high ($>$0.4$R$) \\\midrule
\textbf{UG-FM diff (top-6 HPS-gain prompts)} & $\mathbf{6.76}$ & $5.91$ & $5.72$ \\
random Gaussian noise (control)              & $6.12$ & $6.11$ & $6.11$ \\
\bottomrule
\end{tabular}\end{center}

The random-noise spectrum is essentially flat across bands (max-min $0.012$), confirming the white-noise property. The UG-FM spectrum is monotonically decreasing with frequency and concentrates $+0.64$ nat ($\sim$$1.9{\times}$ in linear magnitude) of additional energy in the low band relative to the high band --- a structured pattern.

\paragraph{Conclusion.} UG-FM is not exploiting a noise-rewarding scorer bug; it is finding a real gradient direction in $z$-space that PickScore and HPS respond to, characterized by a structured low-frequency-dominant perturbation. The structured perturbation pattern is shown at $4{\times}$ zoom on max-diff regions in Fig.~\ref{fig:ug_fm_zoom}.

\begin{figure}[h]
\centering
\includegraphics[width=0.85\linewidth]{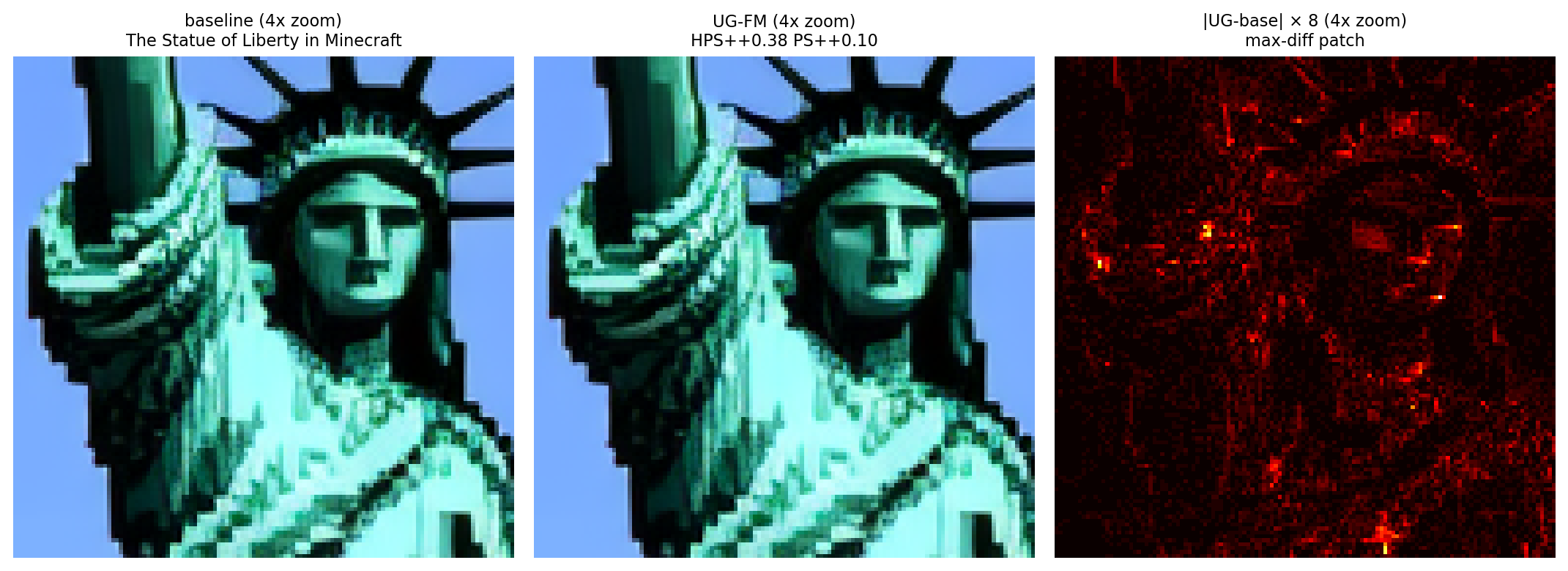}
\caption{$4{\times}$ zoom on the maximum-diff $128{\times}128$ patch from the highest-HPS-gain prompt (``The Statue of Liberty in Minecraft''). Left: baseline (SD3.5 static). Center: UG-FM. Right: $|$UG-baseline$| \times 8$ amplified intensity heatmap. The perturbation localizes in textured / shaded regions, visible at $4$--$8\times$ zoom.}
\label{fig:ug_fm_zoom}
\end{figure}

\section{Related work landscape and extended limitations}
\label{app:related_limits}

\subsection{Inference-time alignment landscape comparison}
\label{app:landscape}

A comparison-matrix view of the prior art discussed in the main paper's Related Work, summarising joint optimization scope, regularization, transport compatibility, T2I scope, and per-step granularity:

\begin{table}[h]
\centering\small
\caption{\textbf{Inference-time alignment landscape.} \cmark{}=present, \xmark{}=absent. PG-MAP is the only framework with all six properties; in particular, it is the only one whose active variable set $\mathcal{A}_t$ is non-trivially time-dependent.}
\label{tab:comparison}
\setlength{\tabcolsep}{3pt}
\begin{tabular}{l|cccccc}
\toprule
Method & Joint $(c,z_t)$ & Forward-cons. & FM-compat. & T2I scope & Per-step & Step-dep.\,$\mathcal{A}_t$ \\
\midrule
UG~\citep{bansal2023universal}        & \xmark & \xmark & limited & \cmark & \cmark & \xmark \\
PNO~\citep{peng2024pno}              & \cmark$^{*}$ & \xmark & \xmark & safety & \xmark & \xmark \\
DATE~\citep{na2025date}              & $c$-only & \xmark & \xmark & \cmark & \cmark & \xmark \\
DNO~\citep{tang2025dno}              & $z$-only & \xmark & \xmark & \cmark & \cmark & \xmark \\
FlowChef~\citep{patel2025flowchef}    & $z$-only & \xmark & \cmark & editing & \cmark & \xmark \\
ReNO~\citep{eyring2024reno}           & noise-only & \xmark & \xmark & \cmark$^{\dagger}$ & \xmark & \xmark \\
\midrule
\textbf{PG-MAP (ours)}                & \cmark & \cmark & \cmark & \cmark & \cmark & \cmark \\
\bottomrule
\end{tabular}
\\\smallskip
\footnotesize{$^{*}$PNO optimizes initial noise $z_T$ + prompt embedding (single trajectory-start perturbation), not per-step $z_t$. $^{\dagger}$ReNO targets one-step distilled T2I models; not applicable to our 28--50-step regime. The \emph{Step-dep.\,$\mathcal{A}_t$} column marks methods whose active variable set $\mathcal{A}_t \subseteq \{c, z_t\}$ varies non-trivially with $t$ (e.g., refine $\{c, z_t\}$ at high-noise steps but $\emptyset$ otherwise on DDPM, or $\{z_t\}$ at data-side only on FM); all prior methods hold $\mathcal{A}_t$ constant across the trajectory.}
\end{table}

\subsection{Extended limitations}
\label{app:limitations_extended}
\paragraph{(L5)~Optimization is non-concave.}
The objective is non-concave due to the denoiser nonlinearity; with $K{=}1$--$2$ steps we obtain only local approximations. The bounded-displacement properties (Appendix~\ref{app:proofs}) are local statements.

\paragraph{(L6)~Compute overhead.}
Wall-clock on SDXL (Tab.~\ref{tab:compute}): MAP-$cz$ runs at $\sim 2.1\times$ baseline; full PG-MAP runs at $\sim 5.5\times$ because the reward backward is unavoidable. Restricts deployment to offline/amortized settings; distillation of $(c_t^\star, z_t^\star)$ via $\pi_\phi$ is the natural follow-up.

\paragraph{(L7)~Reward in-distribution evaluation on SD~1.5.}
On SD~1.5 we use PickScore as both optimisation signal and reported metric (flagged in Tab.~\ref{tab:main_results} via $\dagger$); HPS, CLIPScore, and the human-evaluation study (\S\ref{sec:human_eval}) provide the out-of-distribution evaluation signals.

\section{CRR-MAP details}
\label{app:crr_full}

The main paper Tab.~\ref{tab:crr_results} reports the per-row CRR-MAP oracle results on PartiPrompts ($n{=}1632$, seed 123). This appendix expands the setup, dispatch, oracle-variant ablations, learned-router exploration, FM CRR-MAP, and failure-case breakdown.

\paragraph{Motivating observation.} A 4-prompt SDXL case study (Appendix~\ref{app:cz_extended}, Tab.~\ref{tab:cz_deltas_meanseeds}) shows a prompt-type split: on attribute-binding prompts $c$-optimization is the only variant with non-negative $\Delta$Aes; on atmospheric scenes, reward-driven $z_t$ refinement is the only variant with positive mean $\Delta$Aes. The split motivates the per-prompt routing diagnostic at population scale.

\paragraph{Oracle setup.} We reuse the baseline images and MAP-$cz$ images from Tab.~\ref{tab:main_results} as $f_{\text{base}}$ and $f_{\text{cz}}$, generate the MAP-$c$ images ($f_{\text{c}}$) on the same prompt split, and reuse Tuned-CFG\,$+$\,PG-MAP ($f_{\text{tcfg}}$). All four candidates per prompt are scored with PickScore, HPS\,v2, CLIPScore, and the LAION aesthetic predictor. The oracle is the per-prompt argmax over the four-metric Pareto-sum aggregate (sum of within-method z-scored scores; metric-isolated variants in \S\ref{app:crr_oracle_variants}); it has access to ground-truth scores of each candidate and is the upper bound of any per-prompt selector restricted to the same pool.

\paragraph{Headline numbers and dispatch.} On SDXL $n{=}1632$, oracle (Pareto-sum) routing attains $\mathbf{72.7\%}$ PickScore (paired Wilcoxon $p{=}7.4\!\times\!10^{-88}$), $63.8\%$ CLIPScore ($p{=}4.8\!\times\!10^{-48}$), $\mathbf{73.5\%}$ HPS ($p{=}1.1\!\times\!10^{-93}$), and $68.2\%$ Aesthetic ($p{=}7.9\!\times\!10^{-94}$); on SD~1.5 the ceiling is similarly large ($75.2\%$ / $65.6\%$ / $76.9\%$ / $66.7\%$ on PS / CLIP / HPS / Aes). Simultaneous improvement on all four metrics is an oracle ceiling: the case-study split holds at the population scale (the metric aggregate affects which oracle assignments are made --- pairwise symmetric difference $23.8$--$61.7\%$ across PS-led, CLIP-led, and Pareto-sum aggregates --- but not the qualitative Pareto-improvement signature). The oracle dispatches $32.3\%$ of SDXL prompts to $f_{\text{c}}$, $32.0\%$ to $f_{\text{cz}}$, and $35.7\%$ to $f_{\text{tcfg}}$ (per-prompt assignments in Appendix~\ref{app:crr_per_prompt}). Failure-case breakdown ($\sim 18$\,pp residual degradation rate after Gaussian null adjustment; failure modes dominate: tight attribute binding under high $\lambda$, abstract typography) is in Appendix~\ref{app:failure_breakdown}.

\subsection{CLIP-centroid router formula and lexical overrides}
\label{app:crr_router_full}

A frozen CLIP-text encoder $\phi$ embeds $y$ to $\phi(y)\in\mathbb{R}^{d}$ ($d{=}768$ for ViT-L/14). We curate three small prototype sets $P_{\text{bind}}$, $P_{\text{scene}}$, $P_{\text{bal}}$ ($\approx 10$ prompts each, listed in Appendix~\ref{app:crr_prototypes}) covering attribute-binding, atmospheric scene, and balanced everyday prompts respectively, and define class centroids $\bar{\phi}_k = \mathrm{normalize}\left(\tfrac{1}{|P_k|}\sum_{p\in P_k}\phi(p)\right)$. The base routing decision is
\begin{equation}
k^\star(y) \;=\; \arg\max_{k\in\{\text{bind},\,\text{scene},\,\text{bal}\}}\,\cos\!\big(\phi(y),\,\bar{\phi}_k\big),\qquad
r(y) \;=\;
\begin{cases}
f_{\text{c}}, & k^\star(y)=\text{bind},\\
f_{\text{tcfg}}, & k^\star(y)=\text{scene},\\
f_{\text{cz}}, & k^\star(y)=\text{bal}.
\end{cases}
\label{eq:crr_router}
\end{equation}
Two simple lexical overrides are applied before Eq.~\ref{eq:crr_router}: prompts of $\le 3$ tokens, and prompts containing typography cues (e.g., \texttt{the word ...}, \texttt{a sign reading ...}) are forced to $f_{\text{c}}$. The router cost is one CLIP-text forward pass ($\le 5$\,ms on RTX PRO 6000 Blackwell); in NFE units the router contribution is $\sim 0$.

\subsection{Prototype prompts used by the CLIP-text router}
\label{app:crr_prototypes}

The router of Eq.~\ref{eq:crr_router} compares the input prompt's CLIP-text embedding against three class centroids built from manually-curated prototype prompts, drafted to span the prompt-type axes the case study (\S\ref{sec:cz_analysis}) surfaces.

\begin{center}\small
\begin{tabular}{p{0.30\linewidth} p{0.62\linewidth}}
\toprule
Class & Prototype prompts \\
\midrule
\textsc{bind} (attribute-binding,
geometric, multi-object) & ``a red cube on a blue sphere''; ``a green apple inside a yellow basket''; ``a small blue car next to a large white truck''; ``a glass of orange juice with red straws''; ``the word HELLO in big block letters''; ``a stop sign next to a yield sign''; ``two cats and three dogs''; ``a yellow umbrella next to a blue umbrella''; ``a red triangle on top of a green square''; ``an apple, a banana, and a pear''. \\
\midrule
\textsc{scene} (atmospheric, artistic, landscape, portrait) & ``a serene mountain landscape at golden hour''; ``an oil painting of a stormy sea with crashing waves''; ``a cyberpunk city street in the rain at night''; ``a misty forest with rays of sunlight piercing the canopy''; ``an aerial view of a coral reef in turquoise water''; ``a rolling field of lavender at sunset''; ``a cozy library with ancient books and a fireplace''; ``an art deco hotel lobby''; ``a quiet beach at dawn with seagulls''; ``a Victorian street scene at dusk''. \\
\midrule
\textsc{bal} (everyday, single-subject, casual) & ``a person walking a dog in a park''; ``a chef cooking pasta in a kitchen''; ``a child playing with a toy on a wooden floor''; ``a cat sleeping on a couch''; ``a cup of coffee on a desk''; ``a bicycle leaning against a brick wall''; ``a horse running through a field''; ``a dog catching a frisbee''; ``a woman reading a book''; ``a butterfly on a flower''. \\
\bottomrule
\end{tabular}
\end{center}

The class centroids $\bar{\phi}_k$ are computed once at deployment by averaging the L2-normalized CLIP-text embeddings of each prototype set and re-normalizing.

\subsection{Lexical override rules}
\label{app:crr_lexical_overrides}

Two simple lexical rules apply before Eq.~\ref{eq:crr_router}; both force routing to $f_{\text{c}}$:
\begin{itemize}\setlength\itemsep{1pt}
\item \emph{Short-prompt override.} Prompts of $\le 3$ tokens route to $f_{\text{c}}$. The latent-reward variants over-steer when the prompt admits a wide compatible image manifold.
\item \emph{Typography override.} Prompts containing \texttt{the word}, \texttt{sign that reads}, \texttt{sign reading}, \texttt{letters spelling}, \texttt{text that says}, or \texttt{in big block letters} route to $f_{\text{c}}$. Latent perturbation degrades legibility.
\end{itemize}
The lexical rules are defined a~priori from the prompt-type analysis of Section~\ref{sec:cz_analysis}; they are not tuned on the test split.

\subsection{Oracle variants and metric aggregates}
\label{app:crr_oracle_variants}

The oracle row of Tab.~\ref{tab:crr_results} uses the four-metric aggregate $r^\star(y)=\arg\max_{k}\big(\tilde{\textsc{ps}}(k,y)+\tilde{\textsc{hps}}(k,y)+\tilde{\textsc{clip}}(k,y)+\tilde{\textsc{aes}}(k,y)\big)$, where each tilde is the within-method z-score across the routing pool. We report three metric-isolated variants:

\begin{center}\small
\begin{tabular}{lcccc}
\toprule
Oracle aggregate (SDXL, $n{=}1632$) & PickScore & HPS & CLIPScore & Aesthetic \\
\midrule
PS-only        & $\mathbf{86.3\%}$ & $67.8\%$ & $52.9\%$ & $58.1\%$ \\
CLIP-only      & $55.6\%$ & $58.0\%$ & $\mathbf{81.8\%}$ & $54.8\%$ \\
Pareto-sum (default) & $68.8\%$ & $69.9\%$ & $63.7\%$ & $70.8\%$ \\
Balanced rank  & $73.4\%$ & $\mathbf{74.3\%}$ & $64.3\%$ & $67.5\%$ \\
\bottomrule
\end{tabular}
\end{center}

The four aggregates produce quantitatively different oracles, with pairwise symmetric difference between $23.8\%$ and $61.7\%$. We adopt Pareto-sum as the headline aggregate because it most cleanly demonstrates that no single fixed deployment can match its multi-metric envelope.

\subsection{PartiPrompts \emph{Challenge}-category breakdown}
\label{app:crr_categories}

We partition the $n{=}1632$ test split along the PartiPrompts \emph{Challenge} axis, coarsening into $5$ groups: \textsc{binding}, \textsc{typography}, \textsc{scene}, \textsc{linguistic}, \textsc{general}.

\begin{table}[h]
\centering\small
\caption{PartiPrompts \emph{Challenge}-category breakdown of win rates ($\%$) vs.\ baseline on SDXL ($n{=}1632$, seed 123). The breakdown surfaces a clean prompt-type split that motivates the per-prompt routing of \S\ref{sec:crr_results}: each variant has its own win category --- MAP-$c$ leads CLIP on \textsc{typography}; MAP-$cz$ / PG-MAP lead PickScore on \textsc{general} and \textsc{scene}; and Tuned-CFG\,$+$\,PG-MAP is the recommended HPS deployment, leading HPS on every category. The PG-MAP defaults without Tuned-CFG specialize for PickScore / CLIP / Aesthetic; the deployment trade-off (HPS vs.\ PickScore / CLIP / Aesthetic) is the routing signal CRR-MAP exploits. \emph{Top:} PickScore and HPS. \emph{Bottom:} CLIP and Aesthetic.}
\label{tab:partiprompts_categories}
\setlength{\tabcolsep}{4pt}
% ---- Top half: PickScore + HPS ----
\begin{tabular}{l|cccc|cccc}
\toprule
Category (n) & \multicolumn{4}{c|}{PickScore} & \multicolumn{4}{c}{HPS} \\
 & $c$ & $cz$ & pg & t{+}pg & $c$ & $cz$ & pg & t{+}pg \\
\midrule
Binding ($125$)        & $50.4\%$ & $54.4\%$ & $54.4\%$ & $\mathbf{56.0\%}$ & $58.4\%$ & $48.8\%$ & $48.8\%$ & $\mathbf{69.6\%}$ \\
Typography ($90$)      & $52.2\%$ & $52.2\%$ & $54.4\%$ & $\mathbf{64.4\%}$ & $52.2\%$ & $57.8\%$ & $56.7\%$ & $\mathbf{72.2\%}$ \\
Scene ($422$)          & $54.5\%$ & $55.9\%$ & $\mathbf{56.4\%}$ & $54.3\%$ & $45.5\%$ & $46.7\%$ & $48.3\%$ & $\mathbf{65.6\%}$ \\
Linguistic ($61$)      & $54.1\%$ & $55.7\%$ & $54.1\%$ & $\mathbf{59.0\%}$ & $47.5\%$ & $49.2\%$ & $49.2\%$ & $\mathbf{67.2\%}$ \\
General ($923$)        & $49.7\%$ & $56.7\%$ & $\mathbf{56.8\%}$ & $47.8\%$ & $51.6\%$ & $46.0\%$ & $46.4\%$ & $\mathbf{62.6\%}$ \\
\midrule
\emph{All} ($1632$)    & $51.4\%$ & $56.2\%$ & $\mathbf{56.4\%}$ & $51.3\%$ & $50.3\%$ & $47.2\%$ & $47.9\%$ & $\mathbf{64.6\%}$ \\
\bottomrule
\end{tabular}\\[6pt]
% ---- Bottom half: CLIP + Aesthetic ----
\begin{tabular}{l|cccc|cccc}
\toprule
Category (n) & \multicolumn{4}{c|}{CLIP} & \multicolumn{4}{c}{Aesthetic} \\
 & $c$ & $cz$ & pg & t{+}pg & $c$ & $cz$ & pg & t{+}pg \\
\midrule
Binding ($125$)        & $43.2\%$ & $43.2\%$ & $40.8\%$ & $\mathbf{49.6\%}$ & $52.0\%$ & $50.4\%$ & $51.2\%$ & $\mathbf{60.8\%}$ \\
Typography ($90$)      & $\mathbf{57.8\%}$ & $51.1\%$ & $56.7\%$ & $57.8\%$ & $54.4\%$ & $\mathbf{66.7\%}$ & $66.7\%$ & $64.4\%$ \\
Scene ($422$)          & $49.8\%$ & $50.0\%$ & $\mathbf{51.7\%}$ & $49.1\%$ & $49.5\%$ & $60.2\%$ & $\mathbf{60.4\%}$ & $56.4\%$ \\
Linguistic ($61$)      & $41.0\%$ & $47.5\%$ & $54.1\%$ & $\mathbf{62.3\%}$ & $50.8\%$ & $44.3\%$ & $42.6\%$ & $\mathbf{52.5\%}$ \\
General ($923$)        & $48.4\%$ & $48.4\%$ & $47.8\%$ & $\mathbf{53.6\%}$ & $49.1\%$ & $56.2\%$ & $\mathbf{56.4\%}$ & $55.0\%$ \\
\midrule
\emph{All} ($1632$)    & $48.5\%$ & $48.6\%$ & $49.0\%$ & $\mathbf{52.8\%}$ & $49.8\%$ & $57.0\%$ & $\mathbf{57.2\%}$ & $56.5\%$ \\
\bottomrule
\end{tabular}
\end{table}

\subsection{Per-prompt routing distribution and oracle disagreements}
\label{app:crr_per_prompt}

\begin{center}\small
\begin{tabular}{lccc}
\toprule
Oracle aggregate (SDXL, $n{=}1632$) & $\to f_{\text{c}}$ & $\to f_{\text{cz}}$ & $\to f_{\text{tcfg}}$ \\
\midrule
Pareto-sum (default) & $32.3\%$ ($527$) & $32.0\%$ ($522$) & $35.7\%$ ($583$) \\
PS-led               & $25.9\%$ ($423$) & $38.2\%$ ($623$) & $35.9\%$ ($586$) \\
CLIP-led             & $29.7\%$ ($485$) & $29.4\%$ ($479$) & $40.9\%$ ($668$) \\
Aesthetic-led        & $23.5\%$ ($383$) & $34.4\%$ ($561$) & $42.2\%$ ($688$) \\
\bottomrule
\end{tabular}\end{center}

All four oracle aggregates dispatch a non-trivial mass to each pool member. The four oracle distributions agree on the qualitative pattern (each pool member is informative for some non-trivial subset) but disagree quantitatively (pairwise symmetric difference $23.8$--$61.7\%$).

\subsection{Deployable router heads: explored and future directions}
\label{app:crr_classifier_alts}

The oracle ceiling reported in Tab.~\ref{tab:crr_results} is the upper bound for any selector restricted to the 3-method pool. Building a router head that approaches this ceiling at $\sim 0$ inference-cost overhead is a follow-up direction; preliminary CLIP-prototype (Eq.~\ref{eq:crr_router}) and 5-fold-CV linear-probe routers using only prompt-text features deliver $\sim 1$--$3$\,pp above the best fixed deployment on each metric, indicating that the prompt-text signal alone is insufficient and that approaching the oracle ceiling requires an image-conditioned or learned router. Three follow-up directions:

\begin{itemize}\setlength\itemsep{1pt}
\item \emph{Image-conditioned router.} Generate a single quick-and-dirty image (e.g., the baseline output) and embed it with CLIP-image; concatenate with CLIP-text. The router would have access to image-grounded structure (composition complexity, color palette, texture density).
\item \emph{Per-metric distillation.} Train four metric-specific routers, each predicting ``which method wins on this metric'', and let downstream deployment pick a router based on the prioritised metric.
\item \emph{Zero-shot LLM classifier.} A frozen instruction-tuned LLM with a 3-class system prompt. Adds latency ($\sim 100$\,ms / prompt); valuable when the deployment already has an LLM in the loop.
\end{itemize}

\end{document}